\documentclass[10pt,twocolumn,letterpaper]{article}

\usepackage{cvpr}              

\usepackage{graphicx}
\usepackage{amsmath}
\usepackage{amssymb}
\usepackage{booktabs}
\usepackage{algorithm}
\usepackage[english]{babel}
\usepackage{amsthm}

\usepackage{url}
\usepackage{graphicx}
\usepackage{algorithm}
\usepackage{algorithmic}
\usepackage{color}
\usepackage{amsthm}
\usepackage{makecell}
\usepackage{wrapfig}
\usepackage{graphicx}
\usepackage{float}
\usepackage{xcolor}
\usepackage{verbatim}
\usepackage{multirow}
\usepackage{diagbox}
\usepackage{enumitem}
\usepackage{booktabs}
\usepackage[numbers]{natbib}
\newtheorem{theorem}{Theorem}

\usepackage[pagebackref,breaklinks,colorlinks]{hyperref}
\usepackage{paralist, tabularx}

\usepackage[capitalize]{cleveref}
\crefname{section}{Sec.}{Secs.}
\Crefname{section}{Section}{Sections}
\Crefname{table}{Table}{Tables}
\crefname{table}{Tab.}{Tabs.}

\def\cvprPaperID{*****} 
\def\confName{CVPR}
\def\confYear{2022}

\usepackage{overpic}
\usepackage{enumitem} 
\usepackage{overpic} 
\usepackage{color}

\definecolor{turquoise}{cmyk}{0.65,0,0.1,0.3}
\definecolor{purple}{rgb}{0.65,0,0.65}
\definecolor{dark_green}{rgb}{0, 0.5, 0}
\definecolor{orange}{rgb}{0.8, 0.6, 0.2}
\definecolor{red}{rgb}{0.8, 0.2, 0.2}
\definecolor{darkred}{rgb}{0.6, 0.1, 0.05}
\definecolor{blueish}{rgb}{0.0, 0.3, .6}
\definecolor{light_gray}{rgb}{0.7, 0.7, .7}
\definecolor{pink}{rgb}{1, 0, 1}
\definecolor{greyblue}{rgb}{0.25, 0.25, 1}

\usepackage{blindtext}

\renewcommand{\paragraph}[1]{\vspace{1em}\noindent\textbf{#1}.}

\begin{document}

\title{Towards Efficient and Scalable Sharpness-Aware Minimization}

\author{Yong Liu\textsuperscript{1}, \quad Siqi Mai\textsuperscript{1}, \quad Xiangning Chen\textsuperscript{2}, \quad Cho-Jui Hsieh\textsuperscript{2}, \quad Yang You\textsuperscript{1}\\
\textsuperscript{1}Department of Computer Science, National University of Singapore \\
\textsuperscript{2}Department of Computer Science,  University of California, Los Angeles \\
{\tt\small \{liuyong, siqimai, youy\}@comp.nus.edu.sg},
{\tt\small \{xiangning, chohsieh\}@cs.ucla.edu}
}


\maketitle

\begin{abstract}

Recently, Sharpness-Aware Minimization (SAM), which connects the geometry of the loss
landscape and generalization, has demonstrated significant performance boosts on training large-scale models such as vision transformers. 
However, the update rule of SAM requires two sequential (non-parallelizable) gradient computations at each step, which can double the computational overhead. In this paper, we propose a novel algorithm LookSAM - that only periodically calculates the inner gradient ascent, to significantly reduce the additional training cost of SAM. The empirical results illustrate that LookSAM achieves similar accuracy gains to SAM while being tremendously faster - it enjoys comparable computational complexity with first-order optimizers such as SGD or Adam. 
To further evaluate the performance and scalability of LookSAM, we incorporate a layer-wise modification and perform experiments in the large-batch training scenario, which is more prone to converge to sharp local minima. We are the first to successfully scale up the batch size when training Vision Transformers (ViTs). With a 64k batch size, we are able to train ViTs from scratch in minutes while maintaining competitive performance.
\end{abstract}

\section{Introduction}
\label{sec:intro}
%


It has been observed that sharp local minima usually leads to significantly dropped generalization performance of deep networks, and many methods have been proposed for mitigating this issue~\cite{smith2017bayesian, kwon2021asam, dziugaite2017computing,chaudhari2019entropy,izmailov2018averaging,jin2018local}. 
In particular, \citet{foret2020sharpness} recently  proposed an algorithm named Sharpness Aware Minimization (SAM), which explicitly penalizes the sharp minima and biases the convergence to a flat region. SAM has been used to achieve state-of-the-art performance in many applications. For instance, \citet{chen2021vision} showed that SAM optimizer can improve the validation accuracy of Vision Transformer models (ViTs)~\cite{dosovitskiy2020image} on ImageNet-1k by a significant amount (+5.3\% when training from scratch). However, the update rule of SAM involves two sequential (non-parallelizable) gradient computations at each step, which will double the training time. 

In this paper, we aim to improve the efficiency of SAM and apply it to large-scale training problems. Each step of SAM consists of two gradient computations -- one for  adversarial perturbation to the weights and then another one with the perturbed weights for the final update.  
A naive idea to speedup SAM is to 
compute the first gradient (adversarial perturbation on weights) only periodically and use standard SGD/Adam updates in between. 
Unfortunately, this leads to significantly degraded performance, as shown in our experiments. To resolve this issue, we decompose the SAM's update direction into two components --- the one that lies parallel to the original SGD direction and the other orthogonal component that biases the learning towards a flat region. We show that the second direction tends to remain similar across nearby iterations, both empirically and theoretically, and develop a novel LookSAM optimizer which reuses this direction across nearby iterations. The resulting LookSAM only needs to periodically calculate the inner gradient ascent and significantly reduce the computational complexity of SAM while maintaining similar generalization performance. 

As SAM has become a crucial component for training large-scale Vision Transformer models (ViTs)~\cite{chen2021vision}, 
to further evaluate the performance and scalability of the proposed algorithm, we consider a challenging task --- applying LookSAM to conduct large-batch training for ViTs. As pointed out in \cite{you2017large, you2019large}, large-batch training often introduces the non-uniform instability problem across different layers. Hence, we also adopt a layer-wise scaling rule for weight perturbation, namely Look-LayerSAM optimizer. The proposed optimizer can successfully train ViTs with 64K batch size within an hour while maintaining competitive performance.



Our contributions can be summarized in three folds. 
\begin{compactitem}
    \item We develop a novel algorithm, called LookSAM, to speed up the training of SAM.  
    Instead of computing the inner gradient ascent at every step, the proposed LookSAM computes it periodically and reuses the direction that promotes to flat regions. 
    The empirical results illustrate that LookSAM achieves similar accuracy gains to SAM while enjoying comparable computational complexity with first-order optimizers such as SGD or Adam.
    \item 
    Inspired by the successes of layer-wise scaling proposed in large-batch training~\cite{you2017scaling,you2019large}, we develop an algorithm to scale up the batch size of LookSAM by adopting layer-wise scaling rule for weight perturbation (Look-LayerSAM). 
    The proposed Look-LayerSAM can scale up the batch size to 64k, which is a new record for ViT training and is $16 \times$ compared with previous training settings.
    \item Our proposed Look-LayerSAM can achieve \textbf{$\bf{\sim 8\times}$ speedup} over the training settings in~\cite{dosovitskiy2020image} with a 4k batch size, and we can finish \textbf{the ViT-B-16 training in 0.7 hour}. To the best of our knowledge, this is a new speed record for ViT training.
\end{compactitem}

\section{Related Work}
\label{sec:related}
In this section, we will first review related work on mitigating the sharp local minima problem, and then describe recent developments in large batch training.

\paragraph{Sharp Local Minima}
Sharp local minima would largely influence the generalization performance of deep networks~\cite{smith2017bayesian, kwon2021asam, dziugaite2017computing,chaudhari2019entropy,izmailov2018averaging,jin2018local}. 
Recently, many studies have attempted to explore the studies of sharp local minima, thus to address the optimization problem \cite{yi2019positively,tsuzuku2020normalized,dinh2017sharp,li2017visualizing,chaudhari2019entropy,kwon2021asam,he2019asymmetric,foret2020sharpness}. For example, \citet{jastrzkebski2017three} state that three factors - learning rate, batch size, and gradient covariance, can influence the minima found by SGD. Besides,
\citet{chaudhari2019entropy} propose a local-entropy-based objective function that favors flat regions during training, to avoid approaching the sharp valleys and bad generalization. \citet{he2019asymmetric} observe the loss surfaces and introduce the concept of asymmetric valleys to derive a deeper understanding of flat and sharp minima. By the discovery of Fisher Information Matrix (FIM) as an implicit
regularizer of SGD, \citet{jastrzebski2021catastrophic} try to explicitly penalize the trace of the FIM to solve the problem of sharp minima.
\citet{wen2018smoothout} introduce the \textit{SmoothOut} framework to smooth out sharp minima and thereby improve generalization.
More recently,  Sharpness-Aware Minimization (SAM)~\cite{foret2020sharpness} introduce a novel procedure that can simultaneously minimize loss value and loss sharpness to narrow the generalization gap. It presents rigorous empirical results over a variety of benchmark experiments and achieves state-of-the-art performance. The main focus of this paper is on improving the efficiency and scalability of SAM. 
\paragraph{Large-Batch Training}
Large-batch training is an important direction for distributed machine learning, which can improve the utilization of large-scale clusters and accelerate the training process. 
However, training with a large batch size incurs additional challenges \cite{keskar2016large,hoffer2017train}. \citet{keskar2016large} illustrates that large-batch training is prone to converge to sharp local minima and cause a huge generalization gap. The main reason is that the number of interactions will decrease when scaling up the batch size if we fix the number of epochs.
Traditional methods try to carefully tune the hyper-parameters to narrow the generalization gap, such as learning rate, momentum, and label smoothing \cite{goyal2017accurate,li2017scaling,you2018imagenet,shallue2018measuring}. \citet{goyal2017accurate} propose warmup to better tune the learning rate for training, which tries to increase the learning rate from a small value at the beginning stage and then start to decrease after increased to the target value. Leveraging the warmup training strategy, \citet{goyal2017accurate} can scale up the batch size to 8,192 for ResNet-50 \cite{He_2016_CVPR} on ImageNet-1k \cite{deng2009imagenet}. However, these heuristic approaches cannot be regarded as a principle solution for large-batch training \cite{shallue2018measuring}. 

Recently, to avoid these hand-tuned methods, adaptive learning rate on large-batch training has gained enormous attention from researchers \cite{reddi2018adaptive,reddi2019convergence,zhang2019adam}. Many recent works attempt to use adaptive learning rate to scale the batch size for ResNet-50  on ImageNet \cite{martens2015optimizing,iandola2016firecaffe,akiba2017extremely,smith2017don,devarakonda2017adabatch,codreanu2017scale,you2018imagenet,jia2018highly,osawa2018second,you2019large,yamazaki2019yet}. In particular, \citet{you2017scaling} proposed layer-wise adaptive learning rate algorithm LARS \cite{you2017scaling} to scale the batch size to 32k for ResNet-50. Based on LARS optimizer,
\citet{ying2018image} can finish the ResNet-50 training in 2.2 minutes through TPU v3 Pod \cite{ying2018image}. \citet{liu2021concurrent} use adversarial learning to further scale the batch size to 96k. 
In addition, \citet{you2019large} propose the LAMB optimizer to scale up the batch size when training BERT, resulting in a 76 minutes training time.

\section{Method}
In this section, we will first give an overview of the SAM optimizer and discuss the computational overhead introduced by SAM. The proposed algorithms, including LookSAM and Layer-wise LookSAM will then be introduced in full detail. 

\subsection{Overview of SAM}
\label{section_3 sam}

Let $\mathcal{S} = \{(x_i, y_i)\}_{i=1}^n$ be the training dataset, where each sample $(x_{i}, y_{i})$ follows the distribution $\mathcal{D}$. Let $f(x; \boldsymbol{w})$ be the neural network model with trainable parameter $\boldsymbol{w} \in \mathbb{R}^p$. The loss function corresponding to an input $x_{i}$ is given by $l(f(x_{i}; \boldsymbol{w}), y_{i}) \in R^+$, shortened to $l(x_{i})$ for convenience. The empirical training loss can be defined as $\mathcal{L}_{S} = \frac{1}{n}\sum_{i=1}^{n}l(f(x_{i};\boldsymbol{w}), y_{i})$.
In the SAM algorithm~\citep{foret2020sharpness}, we need to find the parameters whose neighbors within the $\ell_p$ ball have low training loss $\mathcal{L}_{S}(\boldsymbol{w})$ through the following modified objective function: 
\begin{equation}
    \mathcal{L}_{S}^{SAM}(\boldsymbol{w}) =  \max_{\|\boldsymbol{\epsilon}\|_{p}\leq \rho} \mathcal{L}_{S}(\boldsymbol{w}+\boldsymbol{\epsilon}
    ),
\label{loss_sam}
\end{equation}
where $p \geq 0$ and $\rho$ is the radius of  the $\ell_{p}$ ball. 
As calculating the optimal solution of inner maximization is infeasible, 
SAM uses one-step gradient ascent to approximate it: 

\begin{equation}
    \boldsymbol{\hat{\epsilon}}(\boldsymbol{w}) = \rho \nabla_{\boldsymbol{w}}\mathcal{L}_{S}(\boldsymbol{w}) /  \|\nabla_{\boldsymbol{w}}\mathcal{L}_{S}(\boldsymbol{w})\| \approx \arg\max_{\|\boldsymbol{\epsilon}\|_{p}\leq \rho} \mathcal{L}_{S}(\boldsymbol{w}+\boldsymbol{\epsilon}). 
\label{sam_grad}
\end{equation}
Finally, SAM computes the gradient with respect to perturbed model $\boldsymbol{w}+\boldsymbol{\hat{\epsilon}}$ for the update: 
\begin{equation}
    \nabla_{\boldsymbol{w}}\mathcal{L}_{S}^{SAM}(\boldsymbol{w})\approx \nabla_{\boldsymbol{w}}\mathcal{L}_{S}(\boldsymbol{w})|_{\boldsymbol{w}+\boldsymbol{\hat{\epsilon}}}.
\label{sam_calculate}
\end{equation}
However, this update rule involves  two sequential gradient computations at each step based, which will double the computational cost for each update.


\subsection{LookSAM}
\label{section looksam}


The main drawback of SAM lies in its computational overhead. The update rule (Eq~\ref{sam_calculate}) demonstrates that each iteration of SAM needs two sequential gradient computations, one for obtaining $\boldsymbol{\hat{\epsilon}}$ and another for computing the gradient descent update (see Figure~\ref{looksam_visual}). This will double the computational complexity compared to SGD or Adam optimizers. Further, these two gradient evaluations are not parallelizable, which will be a bottleneck in large-batch training. 
However, recent work has demonstrated that SAM yields significant accuracy gain when training vision transformer models~\citep{chen2021vision} (e.g., more than 5\% accuracy improvement when training ImageNet from scratch), and further, SAM's ability to escape from sharp minima is valuable in large-batch training. In particular, \citet{keskar2016large} showed that the main challenge in large-batch training is the convergence to sharp local minima due to insufficient noise in first-order stochastic updates, and SAM is a natural remedy for this problem if it can be conducted efficiently. 
These motivate our work on improving SAM's computational efficiency.

\begin{figure}[h]
\begin{center}
\includegraphics[width=0.8\linewidth]{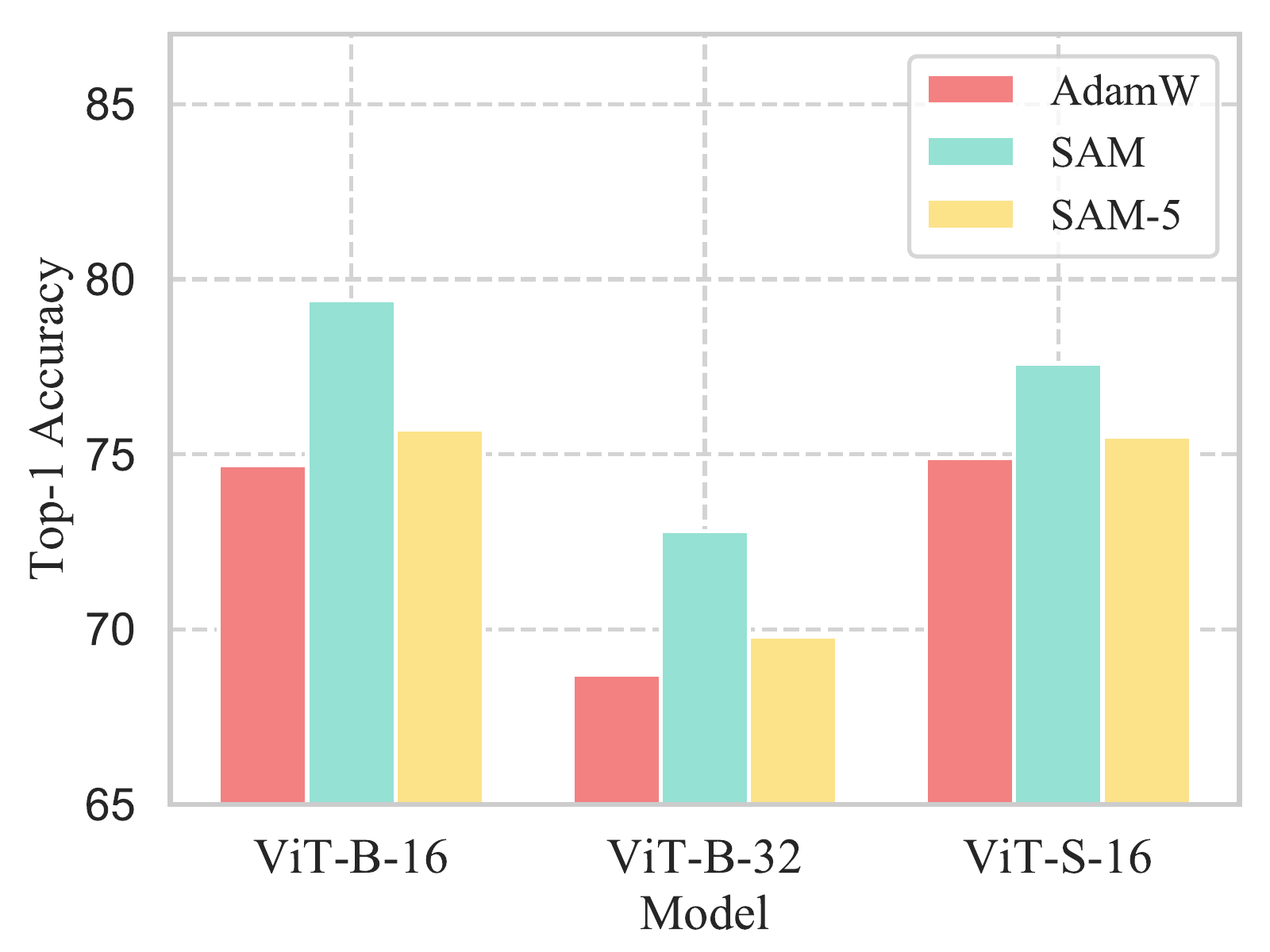}
\end{center}
\caption{Accuracy of SAM-5, SAM and vanilla ViT on ImageNet-1k. SAM-5 indicates the method that calculating SAM gradients every 5 steps.}
\label{sam_analysis}
\end{figure}


\begin{figure}[h]
\begin{center}
\includegraphics[width=0.8\linewidth]{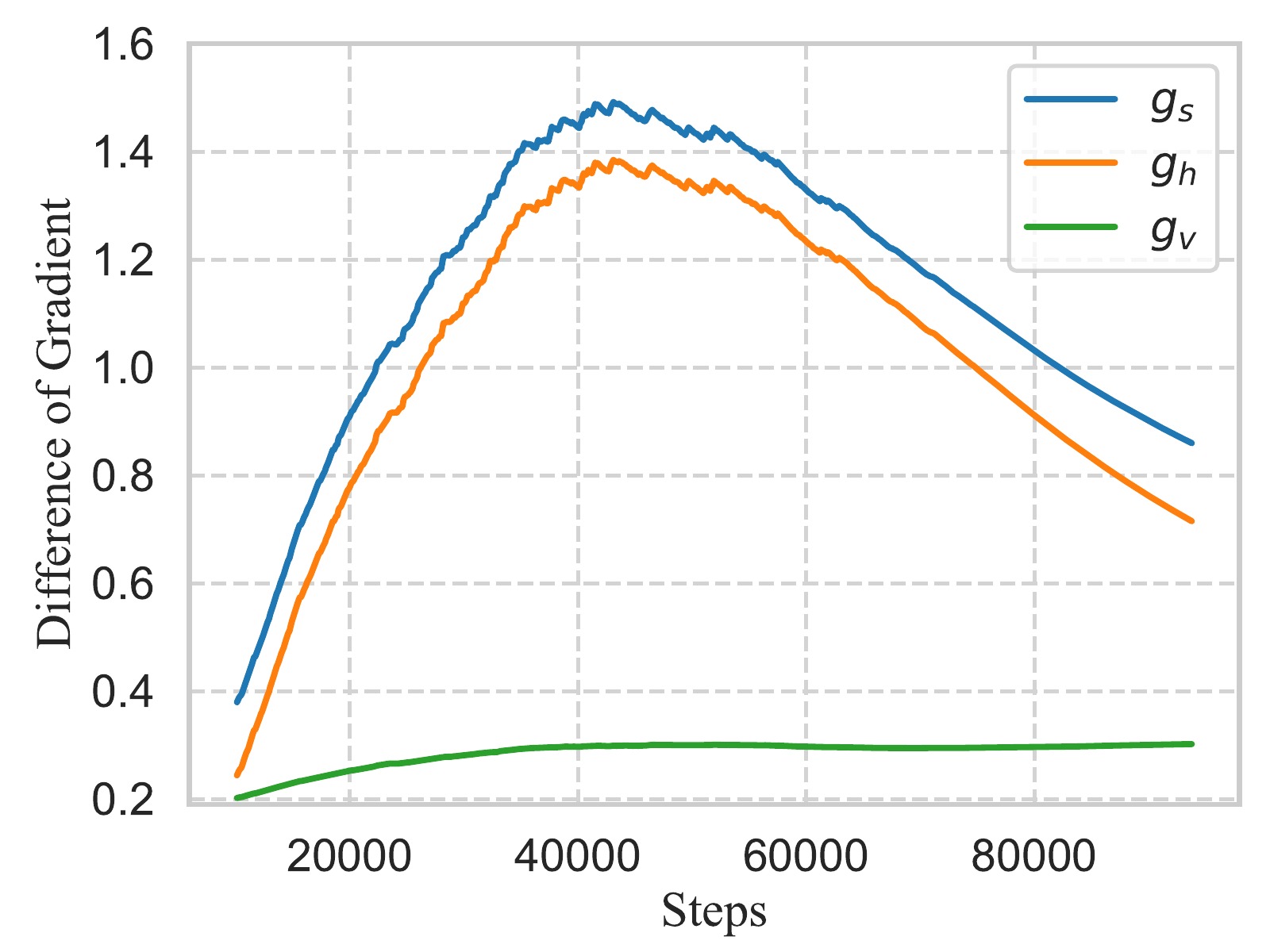}
\end{center}
\caption{Difference of gradients between every 5 steps for $\boldsymbol{g_s}$, $\boldsymbol{g_h}$, and $\boldsymbol{g_v}$ \textcolor{black}{(i.e., $||g_s^t-g_s^{t+k}||$)}. $\boldsymbol{g_v}$ that leads to a smoother region changes much slower than $\boldsymbol{g_s}$ and $\boldsymbol{g_h}$. }
\label{diff_exp}
\end{figure}

To reduce the computation of the two sequential gradients in SAM, a naive method is to use SAM update only at every $k$ step, resulting in $\frac{1}{k}\times$ additional calculation on average. We name this method SAM-$k$, where $k$ indicates the frequency of using SAM. Unfortunately, this naive method does not work well. As shown in Figure \ref{sam_analysis}, we use ViT as the base model and the experimental results illustrate that the accuracy degradation is huge when using SAM-5, although the efficiency is significantly improved. For example, SAM can improve the accuracy from 74.7\% to 79.4\% for ViT-B-16. However, the accuracy drops to 75.7\% when using SAM-5, which significantly degrades the performance of SAM. This motivates us to explore how to effectively improve the efficiency of SAM while maintaining similar generalization performance. 

In the following, we propose a novel LookSAM algorithm to address this challenge, where the main idea is to study how to reuse information to prevent computing SAM's gradient every time. 
As shown in Figure \ref{looksam_visual}, the SAM's gradient $\boldsymbol{g_s} = \nabla_{\boldsymbol{w}}\mathcal{L}_{S}(\boldsymbol{w})|_{\boldsymbol{w}+\boldsymbol{\hat{\epsilon}}}$  promotes to a flatter region (the blue arrow) compared with the SGD gradient (the yellow arrow). To gain more intuition about this flat region, we rewrite the update of SAM based on Taylor Expansion:

\begin{equation}\label{e1}
\begin{aligned}
\nabla_{\boldsymbol{w}}&\mathcal{L}_{S}(\boldsymbol{w})|_{\boldsymbol{w}+\boldsymbol{\hat{\epsilon}}} = \nabla_{\boldsymbol{w}}\mathcal{L}_{S}(\boldsymbol{w}+\boldsymbol{\hat{\epsilon}})\\
&\approx \nabla_{\boldsymbol{w}}[\mathcal{L}_{S}(\boldsymbol{w}) + \boldsymbol{\hat{\epsilon}} \cdot \nabla_{\boldsymbol{w}}\mathcal{L}_{S}(\boldsymbol{w})] \\
&= \nabla_{\boldsymbol{w}}[\mathcal{L}_{S}(\boldsymbol{w}) + \frac{\rho}{\|\nabla_{\boldsymbol{w}}\mathcal{L}_{S}(\boldsymbol{w})\|} \ \nabla_{\boldsymbol{w}}\mathcal{L}_{S}(\boldsymbol{w}) \cdot  \nabla_{\boldsymbol{w}}\mathcal{L}_{S}(\boldsymbol{w})^{T}] \\
&= \nabla_{\boldsymbol{w}}[\mathcal{L}_{S}(\boldsymbol{w}) + \rho \  \|\nabla_{\boldsymbol{w}}\mathcal{L}_{S}(\boldsymbol{w})\|]
\end{aligned}
\end{equation}

\noindent
We find the SAM gradient includes two parts: the original gradient $\nabla_{\boldsymbol{w}}\mathcal{L}_{S}(\boldsymbol{w})$ and the gradient of the L2-Norm of original gradient $\|\nabla_{\boldsymbol{w}}\mathcal{L}_{S}(\boldsymbol{w})\|$. We think optimizing L2-Norm of gradient can prompt the model converge to flat region as the flat region usually means a low gradient norm value. Therefore, the update of SAM can be divided into two parts: the first part (denoted as $\boldsymbol{g_h}$) is to decrease the loss value, and the second part (denoted as $\boldsymbol{g_v}$) is to bias the update to a flat region. More specifically, $\boldsymbol{g_h}$ is in the direction of the vanilla SGD's gradient, which needs to be calculated at each step even without SAM. Therefore, the additional computational cost of SAM is mainly induced by the second part $\boldsymbol{g_v}$. 
Given the SAM's gradient (the red arrow) and the direction of SGD's gradient ($\boldsymbol{g_h}$), we can conduct a projection to obtain $\boldsymbol{g_v}$: 
\begin{equation}
    \label{gv calculation}
    \boldsymbol{g_v} = \nabla_{\boldsymbol{w}}\mathcal{L}_{S}(\boldsymbol{w})|_{\boldsymbol{w}+\boldsymbol{\hat{\epsilon}}} \cdot \sin(\theta),  
\end{equation}
where $\theta$ is the angle between the SGD's gradient and SAM's gradient. 
Empirically, we observe that $\boldsymbol{g_v}$ changes much slower than $\boldsymbol{g_h}$ and  $\boldsymbol{g_s}$. In Figure \ref{diff_exp} we plot the change of these three components between iteration $t$ and iteration $t+5$ throughout the whole training process of SAM, and the results indicate that 
the difference of $\boldsymbol{g_v}$ (the green line) shows a much more stable pattern than that of $\boldsymbol{g_h}$ (the orange line) and $\boldsymbol{g_{s}}$ (the blue line). Intuitively, this means the direction pointing to the flat region won't change significantly within a few iterations. 

\begin{figure}[h]
\begin{center}
\includegraphics[width=0.8\linewidth]{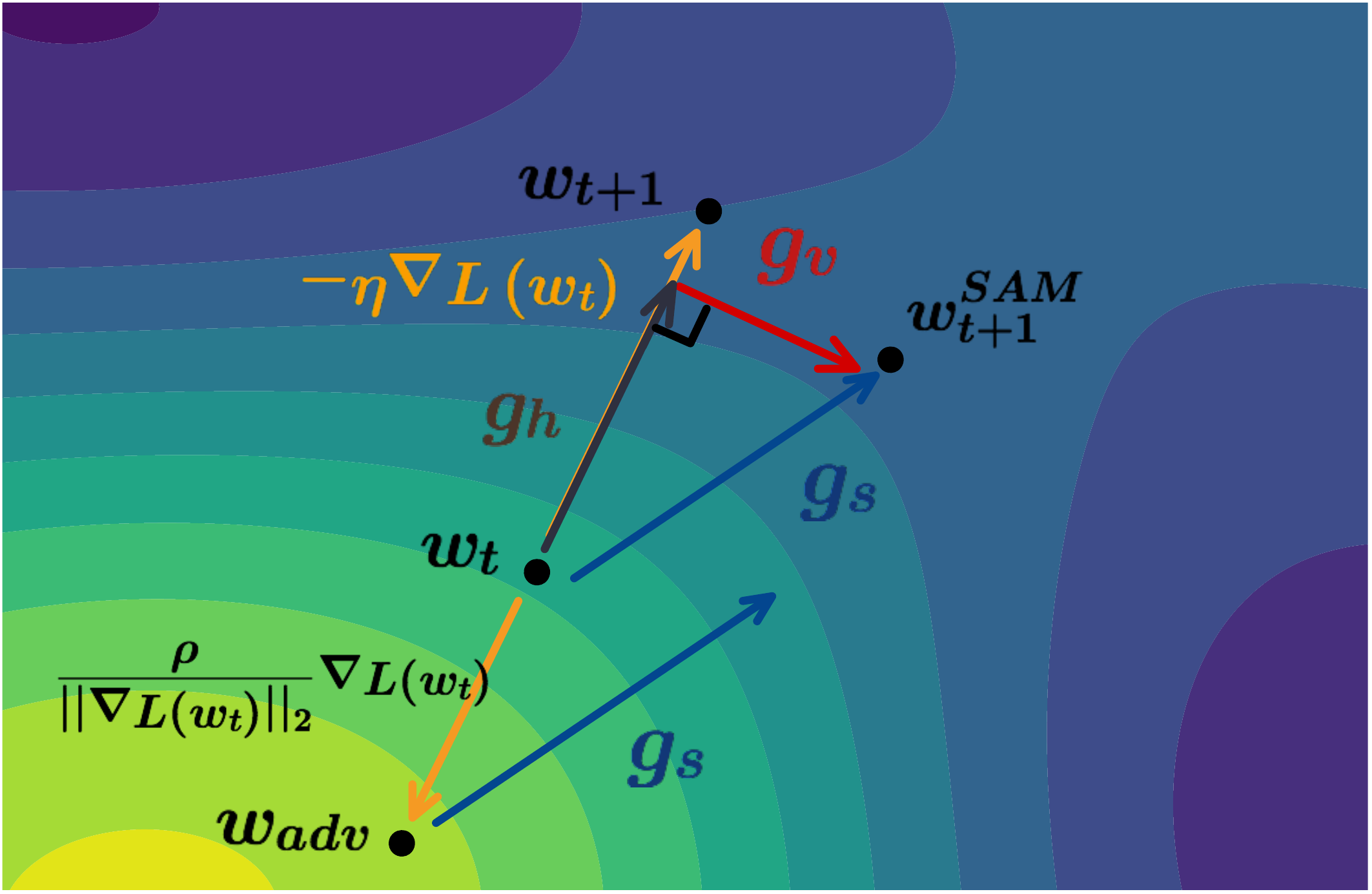}
\end{center}
\caption{Visualization of LookSAM. The blue arrow $\boldsymbol{g_s}$ is SAM's gradient targeting to a flatter region. The yellow arrow $-\eta\nabla_{\boldsymbol{w}}\mathcal{L}_{S}(\boldsymbol{w})$ indicates the SGD gradient. $\boldsymbol{g_h}$ (the brown arrow) and $\boldsymbol{g_v}$ (the red arrow) are the orthogonal gradient components of $\boldsymbol{g_s}$, parallel and vertical to the SGD gradient, respectively.}
\label{looksam_visual}
\end{figure}

Therefore, we propose to only calculate the exact SAM's gradient every $k$ steps and reuse the projected gradient $\boldsymbol{g_v}$ for the intermediate steps. 
The pseudocode is shown in Algorithm \ref{alg:lookSAM}. We calculate the original SGD gradient $\boldsymbol{g}=\nabla_{\boldsymbol{w}}\mathcal{L_B}(\boldsymbol{w})$ based on the sample minibatch $\mathcal{B}$ at every step. For every $k$ step, we compute SAM's gradient and meanwhile get the projected component $\boldsymbol{g_v}$ (Equation \ref{gv calculation})  that will be  reused   for the subsequent steps. At the following $k$ steps, we only calculate the SGD gradient, armed with the projected component to get the approximated SAM gradient. In other words, we train the model and try to mimic the SAM procedure, by sufficiently
distilling the information from SAM gradient every $k$ step. This contributes to the considerable reduction of computation cost, coincident with a smooth convergence that could bias the learning towards a flat region.

To reuse ${\boldsymbol{g_v}}$ in intermediate steps to mimic the SAM's update, we add $\boldsymbol{g_v}$ to the current gradient $\boldsymbol{g}$ (computed on the clean loss). As the empirical analysis in Figure \ref{diff_exp} suggests that  $\boldsymbol{g_s}$ and $\boldsymbol{g_h}$ are not very stable,  we propose an adaptive ratio to combine them. More specifically, we define $\frac{\|\boldsymbol{g}\|}{\|\boldsymbol{g_v}\|}$ as the adaptive ratio to scale $\alpha$. In this way, we can ensure that the norms of $\boldsymbol{g}$ and $\boldsymbol{g_v}$ are at the same scale. 

To demonstrate the reusing procedure, we theoretically derive the change of $\boldsymbol{g_v}$ within $k$ steps in the following way (full derivation is shown in Appendix \ref{deriviation}),

\begin{equation}
\begin{aligned}
\|\boldsymbol{g_{v,t}}& - \boldsymbol{g_{v,t+k}}\| 
\approx\big\|\frac{1}{2}\rho \bigtriangledown_{\boldsymbol{w_t}}^{2}\mathcal{L_S}(\boldsymbol{w_t}) \frac{\bigtriangledown_{\boldsymbol{w_t}}\mathcal{L_S}(\boldsymbol{w_t})}{\|\bigtriangledown_{\boldsymbol{w_t}}\mathcal{L_S}(\boldsymbol{w_t})\|} \\
&- \frac{1}{2}\rho \bigtriangledown_{\boldsymbol{w_{t+k}}}^{2}\mathcal{L_S}(\boldsymbol{w_{t+k}}) \frac{\bigtriangledown_{\boldsymbol{w_{t+k}}}\mathcal{L_S}(\boldsymbol{w_{t+k}})}{\|\bigtriangledown_{\boldsymbol{w_{t+k}}}\mathcal{L_S}(\boldsymbol{w_{t+k}})\|}\big\|, 
\end{aligned}
\end{equation}
where we ignore the third order terms in Taylor expansion in the derivation. 
As calculating SAM gradients leads to a relatively flat region, the second order derivative is very small, from which we can infer the change of $\boldsymbol{g_v}$ component is small compared to the gradient. 
This supports our algorithm in reusing $\boldsymbol{g_v}$ and re-calculating it only periodically.
\\

\begin{algorithm}[tb]
   \caption{LookSAM}
   \label{alg:lookSAM}
\begin{algorithmic}
   \STATE {\bfseries Input:} $x \in \mathbb{R}^{d}$, learning rate $\eta_{t}$, update frequency $k$.
   \FOR{$t \leftarrow 1$ {\bfseries to} $T$}
   \STATE Sample Minibatch $\mathcal{B}=\{(x_i,y_i), \cdots, (x_{|\mathcal{B}|}, y_{|\mathcal{B}|})\}$ from $X$.
   \STATE Compute gradient $\boldsymbol{g} = \nabla_{\boldsymbol{w}}\mathcal{L}_{\mathcal{B}}(\boldsymbol{w})$ on minibatch $\mathcal{B}$.
   \IF{t\%k = 0}
   \STATE Compute $\epsilon(\boldsymbol{w}) = \rho 
   \cdot\nabla_{\boldsymbol{w}}\mathcal{L}_{S}(\boldsymbol{w}) /  \|\nabla_{\boldsymbol{w}}\mathcal{L}_{S}(\boldsymbol{w})\| 
   $
   \STATE Compute SAM gradient: $\boldsymbol{g_{s}} =  \nabla_{\boldsymbol{w}}L_{\mathcal{B}}(\boldsymbol{w})|_{\boldsymbol{w}+\boldsymbol{\epsilon}(\boldsymbol{w})}$
   \STATE $\boldsymbol{g_v} =  \boldsymbol{g_s} - \|\boldsymbol{g_s}\|  \cos(\theta) \cdot \frac{\boldsymbol{g}}{\|\boldsymbol{g}\|}$, where $\cos(\theta) = \frac{\boldsymbol{g}\cdot\boldsymbol{g_s}}{\|\boldsymbol{g}||\boldsymbol{g_s}\|}$ 
   \ELSE
   \STATE $\boldsymbol{g_{s}} = \boldsymbol{g} + \alpha \cdot \frac{\|\boldsymbol{g}\|}{\|\boldsymbol{g_v}\|} \cdot \boldsymbol{g_{v}}$ 
   \ENDIF
   \STATE Update weights: $\boldsymbol{w_{t+1}} = \boldsymbol{w_{t}} - \eta_{t} \cdot \boldsymbol{g_{s}}$
   \ENDFOR
\end{algorithmic}
\end{algorithm}

\subsection{Layer-Wise LookSAM}
\label{section 4.2}

When scaling up the batch size of SAM or LookSAM in large-batch training, we observe degraded performance as shown in the experiments (see Table~\ref{exp_vit_large}).  \citet{you2017scaling,you2019large} showed that the training stability  with large batch training varies for each layer and applied a layer-wise adaptive learning rate scaling method to improve AdamW (also known as LAMB) to resolve this issue. We conjecture this also affects the SAM procedure, which motivates the following development of layer-wise SAM (LayerSAM) optimizer. As we are trying to introduce the layer-wise scaling into the inner maximization of SAM, it is different from \cite{you2019large} which applied the scaling to the final update direction of AdamW. 
\begin{table*}[h]
\renewcommand\arraystretch{1.2}
\caption{Accuracy of Different Models on CIFAR100. We use ResNet-18, ResNet-50 and WideResNet to evaluate the performance of LookSAM, using SGD-Momentum (SGD-M) as the base optimizer. We set the training epoch as 200 and batch size as 128.}
\label{tab_cifar100}
\begin{center}
\begin{tabular}{l|c|cc|cc|cc|c}
\toprule
\multicolumn{1}{c|}{\bf Model} &\multicolumn{1}{c|}{\bf SGD-M}
&\multicolumn{1}{c}{\bf SAM-5}
&\multicolumn{1}{c|}{\bf LookSAM-5}
&\multicolumn{1}{c}{\bf SAM-10}
&\multicolumn{1}{c|}{\bf LookSAM-10} 
&\multicolumn{1}{c}{\bf SAM-20}
&\multicolumn{1}{c|}{\bf LookSAM-20}
&\multicolumn{1}{c}{\bf SAM} 
\\ \hline 
\multicolumn{1}{l|}{\bf ResNet-18}       &78.9   &80.4 & \textbf{80.7} & 80.0  & \textbf{80.4} & 79.7 & \textbf{80.0} & 80.7 \\ 
\multicolumn{1}{l|}{\bf ResNet-50}       &81.4   &82.5 & \textbf{83.3} & 82.3  & \textbf{82.8} & 82.1 & \textbf{82.4} &83.3 \\
\multicolumn{1}{l|}{\bf WRN-28-10}       & 81.7     & 83.8 & \textbf{84.4} & 83.3  & \textbf{84.3} & 82.9 & \textbf{83.6}  & 84.4\\
\bottomrule
\end{tabular}
\end{center}
\end{table*}

Let $\boldsymbol{\Lambda}$ denote a diagonal $l\times{l}$ matrix $\boldsymbol{\Lambda}=\text{diag}(\xi^1,\xi^2,...,\xi^l)$, where $\xi^j$(j=1,2,...,$l$) is the layer-wise adaptive rate and can be calculated by $\frac{\|\boldsymbol{w^j}\|}{\|\bigtriangledown_{\boldsymbol{w}}\mathcal{L_S}(\boldsymbol{w})^j\|}$ for each layer. 

We then adopt this scaling into the inner maximization of SAM as: 

\begin{equation}
   \mathcal{\tilde{L}_S}(\boldsymbol{w})=
    \max_{\|\boldsymbol{\Lambda\epsilon}\|_{p}\leq{\rho}}\mathcal{L_S}(\boldsymbol{w}+\boldsymbol{\epsilon}). 
    \label{LSAM Approx}
\end{equation}
Here the main idea is to scale each dimension of the perturbation vector according $\boldsymbol{\Lambda}$. 
Similar to SAM, the weight perturbation in LayerSAM is the solution of the first-order approximation of~\eqref{LSAM Approx}. With the added $\boldsymbol{\Lambda}$, the approximate inner solution can be written as
\begin{equation}
    \boldsymbol{\tilde{\epsilon}}=\rho\operatorname{sign}(\bigtriangledown_{\boldsymbol{w}}\mathcal{L_S}(\boldsymbol{w}))\boldsymbol{\Lambda}\frac{|\bigtriangledown_{\boldsymbol{w}}\mathcal{L_S}(\boldsymbol{w})|^{q-1}}{(\|\bigtriangledown_{\boldsymbol{w}}\mathcal{L_S}(\boldsymbol{w})\|_q^q)^{\frac{1}{p}}},
    \label{LSAM_epsilon}
\end{equation}
where $\frac{1}{p}+\frac{1}{q}=1$.
Equation \ref{LSAM_epsilon} gives us the layer-wise calculation of $\boldsymbol{\tilde{\epsilon}}$ to scale up the batch size when using LookSAM. 
Algorithm \ref{alg:FastAdv} (in Appendix \ref{psedocode Layersam}) provides the pseudo-code for the full LayerSAM algorithm. Moreover, to combine the advantages of both LookSAM and LayerSAM in large batch training, we further propose Look Layer-wise SAM (Look-LayerSAM) algorithm. 
The pseudo-code is given in Algorithm \ref{alg:FastAdv2}. 
Empirically, we show that Look-LayerSAM significantly outperforms LookSAM in large-batch training, as will be demonstrated in Section \ref{section_exp}.




\section{Experimental Results}
\label{section_exp}
\begin{table*}[ht]
\caption{Top-1 accuracy and training time in per epoch (accuracy/time) of ViTs trained from scratch on ImageNet-1k. We use warmup scheme coupled with a cosine scaling rule for 300 epochs. Following the original setting of ViT, we set batch size as 4,096. }
\label{exp_vit_adamw}
\begin{center}
\renewcommand\arraystretch{1.2}
\begin{tabular}{c|c|cc|cc|c}
\toprule
\multicolumn{1}{c|}{\bf Model}
&\multicolumn{1}{c|}{\bf AdamW} 
&\multicolumn{1}{c}{\bf SAM-5}
&\multicolumn{1}{c|}{\bf LookSAM-5} 
&\multicolumn{1}{c}{\bf SAM-10}
&\multicolumn{1}{c|}{\bf LookSAM-10}
&\multicolumn{1}{c}{\bf SAM} \\ 
\hline 
\multicolumn{1}{c|}{\bf ViT-B-16}       & 74.7/59.7s   & 75.7/68.6s  & \textbf{79.8}/70.5s &75.1/63.7s   & \textbf{78.7}/67.1s & 79.8/103.1s \\
\multicolumn{1}{c|}{\bf ViT-B-32}       & 68.7/21.8s   & 69.8/24.7s  &\textbf{72.6}/26.3s & 69.0/23.4s & \textbf{71.5}/24.4s & 72.8/38.5s \\
\multicolumn{1}{c|}{\bf ViT-S-16}       & 74.9/24.1s    & 75.5/28.3s     &\textbf{77.6}/30.1s &74.9/25.4s &\textbf{77.1}/27.6s & 77.6/44.9s \\
\multicolumn{1}{c|}{\bf ViT-S-32}       &  68.1/18.2s    & 68.7/18.5s   &\textbf{68.8}/19.8s & 68.1/18.5s &\textbf{68.7}/19.5s & 68.9/25.7s \\
\bottomrule
\end{tabular}
\end{center}
\end{table*}
In this section, we evaluate the performance of our proposed LookSAM, LayerSAM, and Look-LayerSAM. First, we empirically illustrate that LookSAM can obtain similar accuracy to vanilla SAM while accelerating the training process. Next, we show that LayerSAM has better generalization for large-batch training on ImageNet-1k compared with vanilla SAM. In addition, we observe Look-LayerSAM can not only scale up to a larger batch size but also significantly speed up the training.
As Vision Transformer (ViT) training has become one of the most important applications of SAM~\cite{chen2021vision}, our experiments will mainly focus on ViT training, while 
we also include some experiments of ResNet and WideResNet on CIFAR100 to further evaluate the generality of the proposed methods. 
\subsection{Setup}

\textbf{Datasets.} To evaluate the efficiency of Look-SAM, we conduct the experiments on  CIFAR-100 \cite{krizhevsky2009learning} and ImageNet-1k \cite{deng2009imagenet} datasets. In addition, ImageNet training is the current benchmark for evaluating the performance of large-batch training \cite{mattson2019mlperf}. In this paper, we also use ImageNet-1k  to train the ViT models. 

\textbf{Models.} \textcolor{black}{We firstly use ResNet-18, ResNet-50 \cite{He_2016_CVPR} and WideResNet \cite{zagoruyko2016wide} to evaluate the performance of Look-SAM on CIFAR-100. To explore the scalability of Look-SAM, we use ViT \cite{dosovitskiy2020image} models to train ImageNet-1k based on the proposed LookSAM optimizer. Finally, we test the performance of our proposed Look-LayerSAM for large-batch training. More specially, we select the ViT models with various sizes to scale up the batch size, such as ViT-Base and ViT-Small for 300 epochs.}

\textbf{Baselines.} Our main baseline is SAM \cite{foret2020sharpness}. To better assess the performance of LookSAM, we propose the algorithm SAM-k as the baseline for comparison. More specially, SAM-k can be seen as the method that directly uses SAM every $k$ step.  

\textbf{Implementation Details.} We implement our algorithm in JAX \cite{jax2018github} and follow the original setting from SAM \cite{foret2020sharpness}. 
To compare the performance of LookSAM with vanilla SAM, we adopt AdamW \cite{loshchilov2017decoupled} as the base optimizer. Note that the input resolution is 224, which is the official setting for ViT. To scale up the batch size, we use LAMB \cite{you2019large} as our base optimizer for large-batch training and compare our approaches with SAM. We apply learning rate warmup scheme \cite{goyal2017accurate} to avoid the divergence due to the large learning rate, where training starts with a smaller learning rate $\eta$ and gradually increases to the large learning rate $\eta$ for 300 epochs.
In addition, to further improve the performance of large-batch training, we use RandAug \cite{cubuk2020randaugment} and Mixup \cite{zhang2017mixup} to scale the batch size to 64k. The implementation details can be found in Appendix \ref{parameter_setting}.

\subsection{CIFAR Training on ResNet and WideResNet}

In this section, we conduct  experiments for training ResNet and WideResNet on CIFAR-100 to evaluate the performance of our proposed algorithms. The experimental results are shown in 
Table \ref{tab_cifar100}. We can find that LookSAM-k can achieve a similar accuracy compared with SAM which is much better than SAM-k. As shown in Table \ref{tab_cifar100}, LookSAM-5 surprisingly achieve the same accuracy as SAM did (80.7\%, 83.3\%, 84.4\%) but with much less training time based on the performance of all the three models, ResNet-18, ResNet-50 and WRN-28-10. Additionally, LookSAM-k shows a remarkable improvement over the performance on SAM-k that likewise takes the comparable training time. Specifically, LookSAM-5 can obtain noticeably higher accuracy (80.7\%, 83.3\%, 84.4\%) compared with SAM-5 (80.4\%, 82.5\%, 83.8\%) on ResNet-18, ResNet-50 and WRN-28-10 respectively. Although the performance of LookSAM-k would degrade as we increase the value k, LookSAM-k remains dominating over SAM-k. For instance, according to the experiment on WRN-28-10, the improvement of LookSAM-k over SAM-k is 
desirable, with an increment of 0.6\%, 1.0\% and 0.7\% for $k=5,10,20$.

The empirical results in Table \ref{tab_cifar100} also demonstrate that 
the performance gap between LookSAM and SAM enlarges 
as model size increases. For example, we can observe an obvious increment of the average improvement of LookSAM-k over SAM-k's when comparing the experiments of ResNet-18 with those of ResNet-50 and WRN-28-10, from 0.37\% to 0.53\% and 0.77\%. Therefore, to further evaluate the performance and scalability of LookSAM, we present the experiment of ImageNet training from scratch on ViT with LookSAM in Section \ref{Vit LookSAM}.


\begin{table*}[t]
\renewcommand\arraystretch{1.2}
\caption{Accuracy of ViT-B-16 on ImageNet-1k for 300 epoch when using RandAug and Mixup. Look-LayerSAM can obtain above 75\% accuracy when we scale up the batch size to 64k.}
\label{exp_large_64k}
\begin{center}
\begin{tabular}{lcccccc}
\toprule
\multicolumn{1}{c}{\bf Model} 
& \multicolumn{1}{c}{\bf Algorithm}
&\multicolumn{1}{c}{\bf RandAug}
&\multicolumn{1}{c}{\bf Mixup}
&\multicolumn{1}{c}{\bf Optimizer}
&\multicolumn{1}{c}{\bf 32k}
&\multicolumn{1}{c}{\bf 64k}
\\ \hline
\multicolumn{1}{l}{\bf ViT-B-16} &\multicolumn{1}{l}{\bf Vanilla ViT}        & & & LAMB & 72.4 & 68.1 \\
\multicolumn{1}{l}{\bf ViT-B-16} &\multicolumn{1}{l}{\bf Look-LayerSAM}   & &     & LAMB & 77.1 & 72.0 \\
\multicolumn{1}{l}{\bf ViT-B-16} &\multicolumn{1}{l}{\bf Look-LayerSAM} & \checkmark &  & LAMB & 79.2  & 74.9 \\
\multicolumn{1}{l}{\bf ViT-B-16} &\multicolumn{1}{l}{\bf Look-LayerSAM} & \checkmark & \checkmark & LAMB & \textbf{79.7}  & \textbf{75.6} \\ 
\bottomrule
\end{tabular}
\end{center}
\end{table*}

\subsection{ImageNet Training from Scratch on Vision Transformer}
\label{Vit LookSAM}
Following the original setting of ViT, we train ViT with LookSAM and compare it with vanilla ViT and SAM-$k$. The experimental results are given in Table \ref{exp_vit_adamw}. 
It shows that LookSAM achieves similar accuracy with vanilla SAM and obtains much better performance than SAM-$k$. 
Specifically, compared with the minimal improvement of SAM-k over vanilla AdamW, LookSAM yields considerable improvements, such as the top-1 accuracy improvement from 74.7\% to 79.8\% on LookSAM-5 ($\uparrow$ 5.1\%), while SAM-5 can only achieve 75.7\%. There is a remarkable improvement ($\uparrow$ 4.1\%) of LookSAM-5 in test accuracy (79.8\%) in comparison to SAM-5 (75.7\%). 
Further, by computing SAM's update only periodically, our methods significantly improve the time cost over SAM while keeping similar predictive performance. 
For instance, LookSAM-5 enables a competitive reduction of training time by 2/3 for ViT-B-16 (from 103.1s to 68.6s) without any loss in test accuracy (79.8\%). Moreover, this advantage is widely reflected in different settings (shown in Table \ref{exp_vit_adamw}) and thereby our proposed methods can be adopted in a variety of ViT models. 

\subsection{Large-Batch Training for Vision Transformer}
In addition to standard training tasks, we further apply the proposed methods to the challenging large-batch distributed training. It has been observed that large-batch training usually converges to sharp local minima with degraded generalization performance~\cite{keskar2016large,goyal2017accurate}. This is due to insufficient noise in gradient estimation and the decreased number of updates. Therefore, scaling an algorithm to large-batch training is a challenging task.  

As mentioned in Section 3.3, we extended LookSAM to Look-LayerSAM to overcome the training instability problem in large-batch training. 
To evaluate the performance of our proposed algorithms for large-batch training, we use Look-LayerSAM to scale the batch size for ViT training on ImageNet-1k. As shown in Table \ref{exp_vit_large}, based on Look-LayerSAM, we can scale the batch size from 4,096 to 32,768 while keeping the accuracy above 77\%. Note that although vanilla SAM can improve the performance of ViT while scaling up, the improvement is weakened as batch size increases. For instance, the improvements are 4\%, 4\%, 3.2\%, 2.7\% from batch size 4,096 to 32,768 over LAMB (which is a standard optimizer for large batch training).  In contrast, our proposed Look-LayerSAM can consistently achieve a higher improvement even if scaling up the batch size to 32,768. 
\begin{table}[h]
\caption{Large-batch training accuracy of ViT-B-16 on ImageNet-1k. We use warmup scheme coupled with linear rule to scale the learning rate for 300 epochs. Look-LayerSAM achieves consistent higher accuracy than SAM from 4k to 32k. }
\label{exp_vit_large}
\renewcommand\arraystretch{1.2}
\begin{center}
\begin{tabular}{lccccc}
\toprule
\multicolumn{1}{c}{\bf Algorithm} &\multicolumn{1}{c}{\bf 4k} &\multicolumn{1}{c}{\bf 8k}   &\multicolumn{1}{c}{\bf 16k}   &\multicolumn{1}{c}{\bf 32k}
\\ \hline
\multicolumn{1}{l}{\bf LAMB}              &74.6    &74.3  & 74.4 & 72.4\\
\multicolumn{1}{l}{\bf LAMB + SAM}        &78.6    &78.3    &77.6   &75.1 \\
\multicolumn{1}{l}{\bf LAMB + Look-SAM}    &78.9   &78.4  &77.1 &75.3 \\
\multicolumn{1}{l}{\bf LAMB + Look-LayerSAM}  &\multicolumn{1}{c}{\bf 80.3} & \multicolumn{1}{l}{\bf 79.5}   & \multicolumn{1}{l}{\bf 78.4}  & \multicolumn{1}{l}{\bf 77.1}  
\\ \bottomrule
\end{tabular}
\end{center}
\end{table}
In particular, the increments on accuracy are stable from 4,096 to 32,768: 5.6\%, 5.8\%, 4.4\%, and 5.5\% over the LAMB optimizer. Moreover, LookSAM is able to achieve the performance on par with the vanilla SAM, while enjoying similar computational cost as LAMB. For example, top-1 accuracy of SAM and LookSAM are 78.6\% and 78.9\%, respectively, when batch size is 4,096. We continue to observe that Look-LayerSAM offers much more considerable benefits on large batch training, including 80.3\% accuracy on 4,096, as well as 77.1\% on batch size 32,768, in which SAM and LookSAM achieve 75.1\% and 75.3\%.

In addition, related work has shown that data augmentation can improve the performance of large-batch training. Therefore, we try to further scale the batch size to 64k based on RandAug and Mixup. The experimental results are shown in Table \ref{exp_large_64k}, which illustrates that our proposed Layer-LookSAM can work together with data augmentation and improve the performance of large-batch training. For instance, Look-LayerSAM can also achieve 74.9\% when applying RandAug and Mixup at 64k. After using Mixup, the accuracy improves to 75.6\%.  

To further evaluate the performance of LookSAM on accelerating the training of SAM, we analyze their training time when scaling batch size from 4,096 to 32,768. Note that we use 128, 256, 512 and 1024 TPU-v3 chips to report the speed of ViT-B-16 on batch size 4,096, 8,192, 16,384, and 32,768. Besides, we use warmup schedule coupled with linear learning rate decay for 300 epochs.
The experimental results are shown in Table \ref{exp_time}, which illustrates that LayerSAM will cause about $1.7\times$ training time compared with vanilla LAMB. However, Look-LayerSAM can significantly reduce the training time and achieve $1.5\times$ speed compared with LayerSAM when $k=5$. In particular, the training time of ViT-B-16 on ImageNet-1k can be reduced to 0.7 hour.
\begin{table}[ht]
\renewcommand\arraystretch{1.2}
\caption{Training Time of ViT-B-16 on ImageNet-1k. We set LAMB as the base optimizer and 300 as the training epoch. We can finish the ViT training within 1 hour.}
\label{exp_time}
\begin{center}
\begin{tabular}{lcccc}
\toprule
\multicolumn{1}{c}{\bf Algorithm} &\multicolumn{1}{c}{\bf 4k} &\multicolumn{1}{c}{\bf 8k}   &\multicolumn{1}{c}{\bf 16k}   &\multicolumn{1}{c}{\bf 32k}
\\ \hline 
\multicolumn{1}{l}{\bf LAMB }             & \textbf{4.8h}    & \textbf{2.4h}  & \textbf{1.2h}  & /   \\
\multicolumn{1}{l}{\bf LAMB + LayerSAM}  & 8.4h  & 4.3h   & 2.2h   & 1.1h   \\
\multicolumn{1}{l}{\bf LAMB + Look-LayerSAM}  & 5.6h & 2.8h    & 1.4h    & \textbf{0.7h}  
\\ \bottomrule
\end{tabular}
\end{center}
\end{table}

To sum up, with Look-LayerSAM, we are able to train Vision Transformer in 0.7 hour and achieve 77.1\% top-1 accuracy on ImageNet-1k with 32K batch size, outperforming existing optimizers such as LAMB and SAM.  

\begin{figure*}[tbp]
\includegraphics[width=5.5cm,height=4cm]{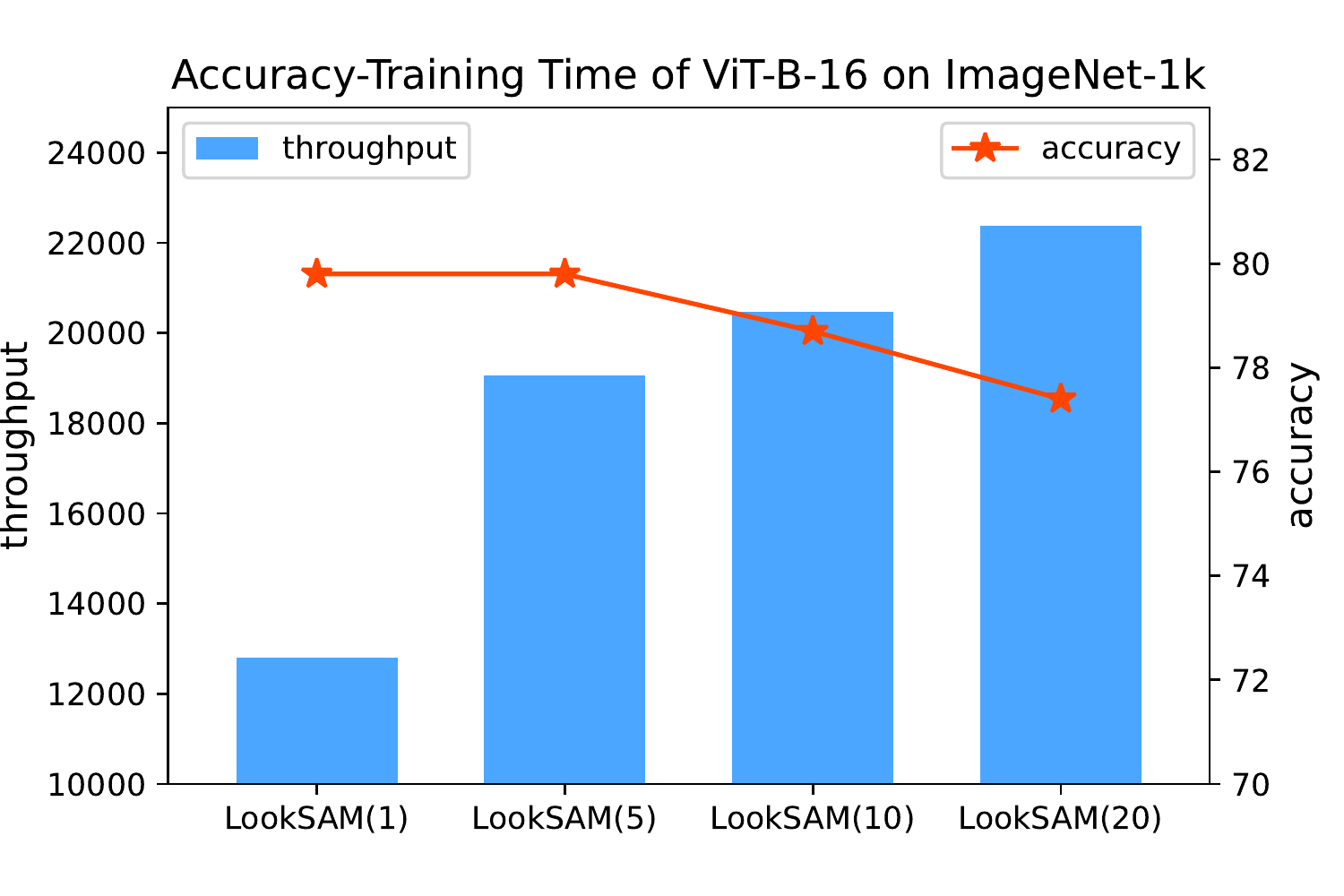}
\hspace{1ex} 
\includegraphics[width=5.5cm,height=4cm]{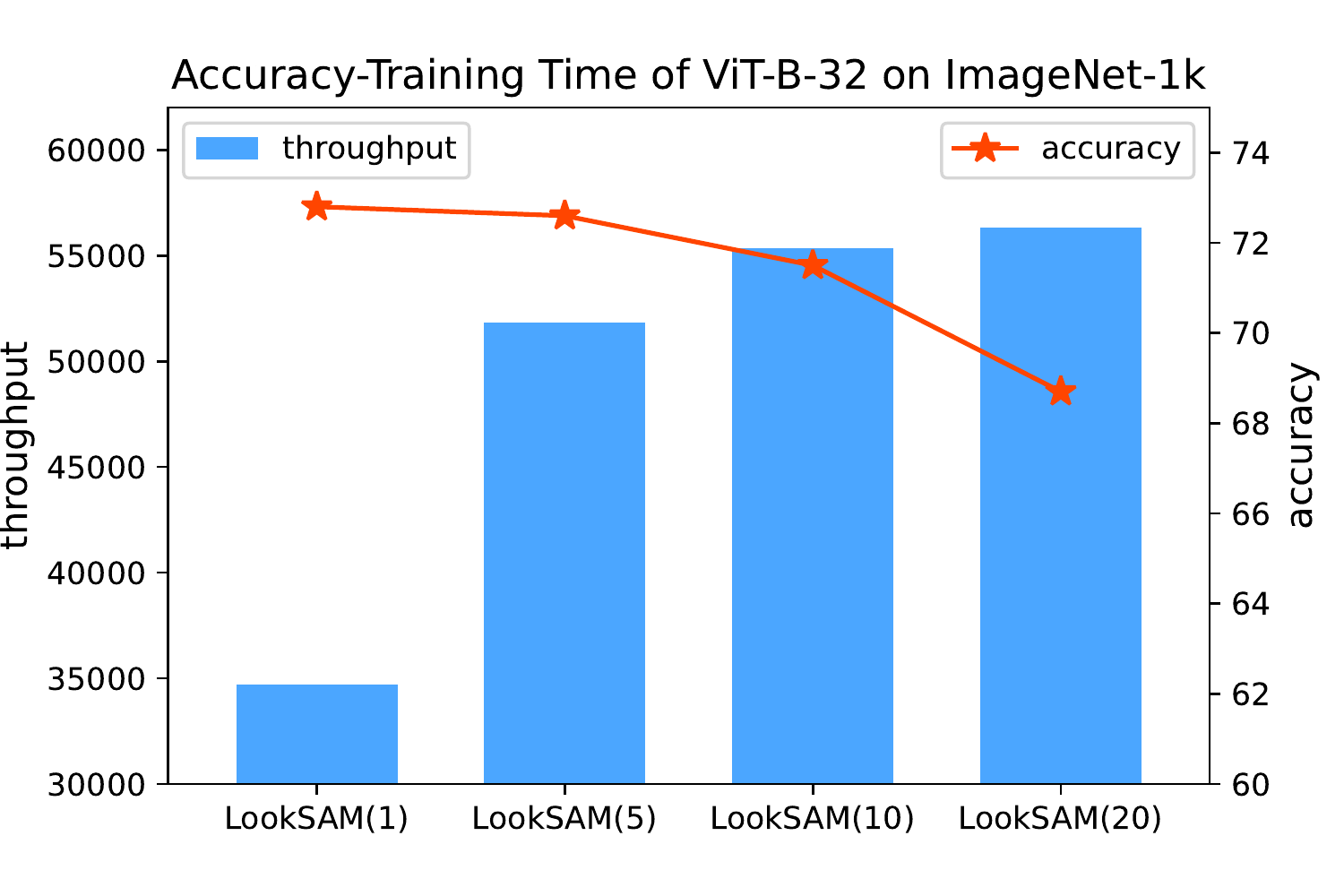}
\hspace{1ex}
\includegraphics[width=5.5cm,height=4cm]{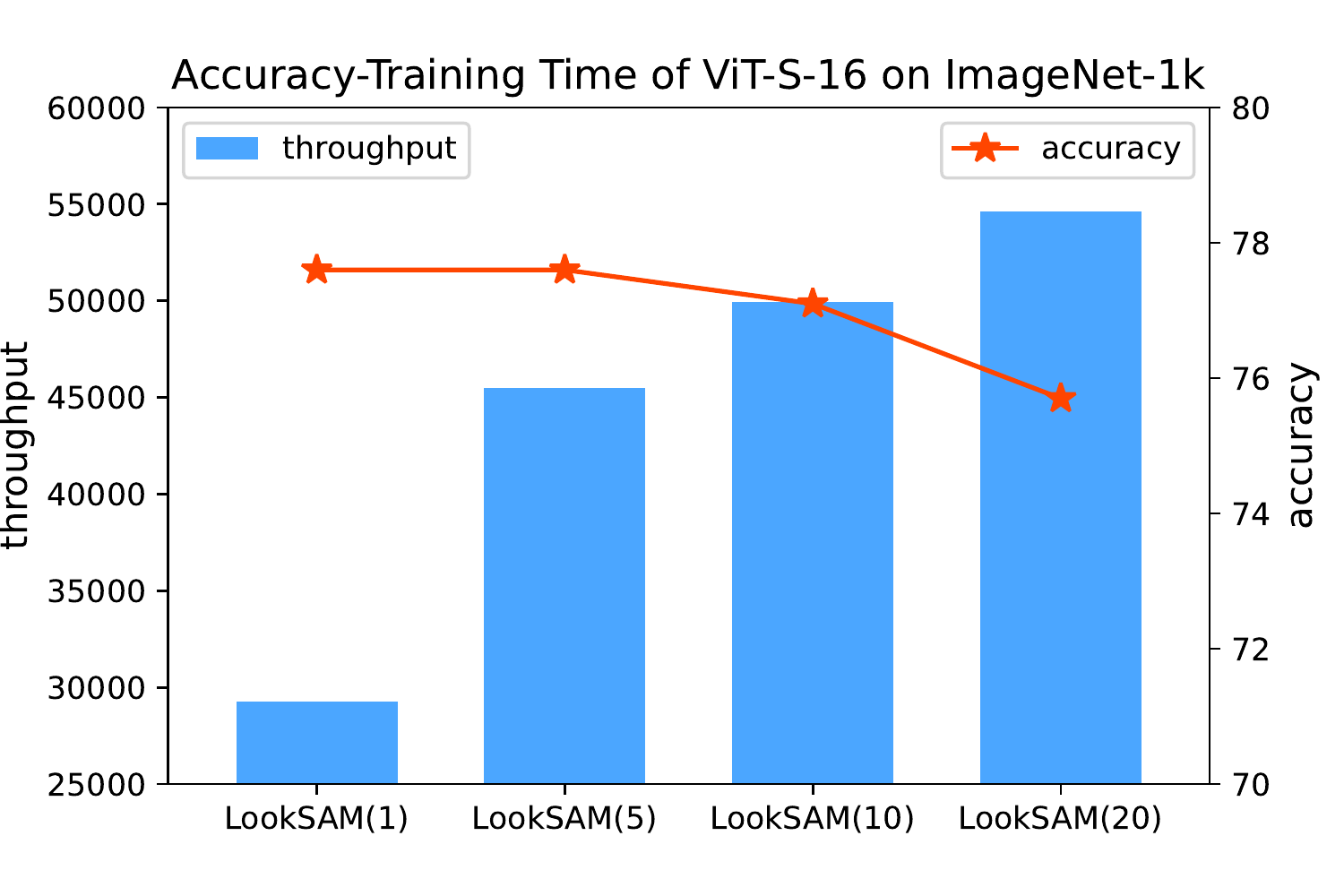}
\caption{Accuracy-Training Time of different models for LookSAM-$k$ on ImageNet-1k. With the growth of $k$ value, the throughput is increasing but the accuracy starts to drop. There is a trade-off between the accuracy and training speed. Note that LookSAM-1 is the same as the original SAM.}
\label{fig:my_label}
\end{figure*}

\subsection{Accuracy and Efficiency Tradeoff}

The reuse frequency $k$ controls the trade-off between accuracy and  speed. In this section, we try to conduct an analysis on the performance of LookSAM with different values of $k$ to further explore the trade-off. The experimental results in Figure \ref{fig:my_label} indicates  that LookSAM can achieve the similar accuracy as vanilla SAM when $k \leq 5$. With reuse frequency $k$ getting larger, the accuracy begins to drop while the training speed is 
accelerated. For example, as shown in Figure \ref{fig:my_label}, the accuracy of LookSAM-5 on ViT-B-16 is 79.8\%, which is the same as the original SAM. In the meantime, the throughput increases from 12,800 (SAM) to 19,051 (LookSAM-5). In addition, when the value $k$ increases to 10, the accuracy drops to 78.7\% (still improves by 4\% compared with AdamW) but the throughput increases to 20,480.

When $k$ is larger than 10, we notice that the speed is converged (almost identical to plain AdamW). 
Therefore, in practice, we can determine the $k$ value based on the desired trade-off, and we recommend $k=5$ for general applications since it will significantly improve the efficiency while still achieving almost equivalent test accuracy as SAM. 
In addition, our proposed LookSAM also provides more selections for deep learning researchers. If the application scenario requires a higher training speed, we can try increasing the frequency $k$. Otherwise, the frequency $k$ can be reduced. 

\subsection{Sensitivity Analysis about Hyper-Parameters}

\subsubsection{Sensitivity Analysis of $\alpha$}
\label{alpha analysis}
We study the effect of gradient reuse weight $\alpha$ has on the performance of training ImageNet-1k. 
We conduct this experiment with batch size 16,384 and 32,768 since large-batch training is usually more sensitive to hyperparameters. The experiments are conducted on ViT-B-16 using Look-LayerSAM, with LAMB as optimizer, and we set $\rho$ as 1.0. We report the validation accuracy for different $\alpha$ (0.5, 0.7, 1.0) in Table \ref{tab_alpha}. When $\alpha=0.7$, Look-LayerSAM achieves the best accuracy 78.4\% on batch size 16,384 and 77.1\% on batch size 32,768. Further, even if $\alpha$ is not well-tuned, Look-LayerSAM is able to obtain a good performance, including above $77\%$ accuracy on 16,384 batch size and $\sim 76\%$ accuracy on 32,768 batch size. 
\begin{table}[h]
\renewcommand\arraystretch{1.2}
\caption{Sensitivity Analysis of $\alpha$. We select ViT-B-16 as our base model and the optimizer is Look-LayerSAM (based on LAMB). } 
\label{tab_alpha}
\begin{center}
\vspace{-10pt}
\begin{tabular}{lccc}
\toprule
\multicolumn{1}{l}{\bf Batch Size}
&\multicolumn{1}{c}{\bf $\alpha$ = 0.5}
&\multicolumn{1}{c}{\bf $\alpha$ = 0.7}
&\multicolumn{1}{c}{\bf $\alpha$ = 1.0} 
\\ \hline 
\multicolumn{1}{l}{\bf 16384}  & 77.7   & \textbf{78.4}  & 78.2  \\
\multicolumn{1}{l}{\bf 32768}  &76.5   & \textbf{77.1}  & 75.9   
\\ \bottomrule
\end{tabular}
\end{center}
\end{table}

\subsubsection{Sensitivity Analysis of $\rho$}

Finally, we conduct a sensitivity analysis for different values of $\rho$, the intensity of perturbation in SAM and LookSAM. We evaluate the accuracy of ViT-B-16 on batch size 16,384 and 32,768. We set $\alpha=0.7$, the best value in our analysis from Section \ref{alpha analysis}. The experimental results regarding $\rho$ (0.5, 0.8, 1.0, 1.2) are shown in Table \ref{rho-table}. We report when $\rho=1.0$, Look-LayerSAM achieves the highest accuracy on both batch size 16,384 (78.4\%) and 32,768 (77.1\%). Additionally, we observe the
overall robustness from the analysis of $\rho$, which gives us $77\%$ accuracy on 16,384 batch size and more than $75\%$ accuracy on 32,768 batch size without finetuning. 
\begin{table}[h]
\renewcommand\arraystretch{1.2}
\caption{Sensitivity Analysis of $\rho$. We select ViT-B-16 as our base model and the optimizer is Look-LayerSAM (based on LAMB). } 
\label{rho-table}
\begin{center}
\vspace{-15pt}
\begin{tabular}{lcccc}
\toprule
\multicolumn{1}{l}{\bf Batch Size}
&\multicolumn{1}{c}{\bf $\rho$ = 0.5}
&\multicolumn{1}{c}{\bf $\rho$ = 0.8}
&\multicolumn{1}{c}{\bf $\rho$ = 1.0} 
&\multicolumn{1}{c}{\bf $\rho$ = 1.2}
\\ \hline
\multicolumn{1}{l}{\bf 16384}   & 77.0  & 77.8   & \textbf{78.4}  & 77.9 \\
\multicolumn{1}{l}{\bf 32768}   & 75.2  & 76.4  & \textbf{77.1}  & 76.7 \\
 \bottomrule
\end{tabular}
\vspace{-10pt}
\end{center}
\end{table}


\section{Conclusions}
\label{section 6}
We propose a novel algorithm LookSAM to reduce the additional computation from SAM and speed up the training process.  To further evaluate the performance in large-batch training, we propose Look-LayerSAM, which uses a layer-wise schedule to scale the weight perturbation of LookSAM. Finally, we evaluate our proposed algorithms on Vision Transformer. Experimental results illustrate that we can scale the batch size to 64k and obtain an accuracy above 75\%. Further, we can achieve about 8$\times$ speedup over the training settings in~\cite{dosovitskiy2020image} with a 4k batch size and finish the ViT training in 0.7 hour. To the best of our knowledge, this is a new speed record for ViT training. 

\section{Acknowledgements}

We thank Google TFRC for supporting us to get access to the Cloud TPUs. We thank CSCS (Swiss National Supercomputing Centre) for supporting us to get access to the Piz Daint supercomputer. We thank TACC (Texas Advanced Computing Center) for supporting us to get access to the Longhorn supercomputer and the Frontera supercomputer. We thank LuxProvide (Luxembourg national supercomputer HPC organization) for supporting us to get access to the MeluXina supercomputer.







{
    \small
    \bibliographystyle{plainnat}
    \bibliography{macros,main}
}

\clearpage

\appendix

\setcounter{page}{1}

\twocolumn[
\centering
\Large
\textbf{Towards Efficient and Scalable Sharpness-Aware Minimization} \\
\vspace{0.5em}Supplementary Material \\
\vspace{1.0em}
] 
\appendix




\section{Appendix}

\subsection{Theoretical Analysis of Projected Gradient}
\label{deriviation}
For SAM loss function $\mathcal{L}_{S}(\boldsymbol{w}+\boldsymbol{\hat{\epsilon}})$, based on Taylor Expansion, we can obtain:

\begin{equation}
    \mathcal{L}_{S}(\boldsymbol{w}+\boldsymbol{\hat{\epsilon}}) 
    \approx \mathcal{L}_{S}(\boldsymbol{w}) + \boldsymbol{\hat{\epsilon}} \nabla_{\boldsymbol{w}}\mathcal{L}_{S}(\boldsymbol{w})
\end{equation}

Therefore, we can rewrite Equation \ref{sam_calculate} as follows:

\begin{equation}
\label{equation looksam}
\begin{aligned}
    \nabla_{\boldsymbol{w}}\mathcal{L}_{S}(\boldsymbol{w})|_{\boldsymbol{w}+\boldsymbol{\hat{\epsilon}}} &= \nabla_{\boldsymbol{w}}\mathcal{L}_{S}(\boldsymbol{w}+\boldsymbol{\hat{\epsilon}}) \\
    &\approx \nabla_{\boldsymbol{w}} 
    [\mathcal{L}_{S}(\boldsymbol{w}) + \boldsymbol{\hat{\epsilon}}\nabla_{\boldsymbol{w}}\mathcal{L}_{S}(\boldsymbol{w})] \\
    &= \nabla_{\boldsymbol{w}}\mathcal{L}_{S}(\boldsymbol{w}) + \boldsymbol{\hat{\epsilon}}\nabla_{\boldsymbol{w}}^{2}\mathcal{L}_{S}(\boldsymbol{w})
\end{aligned}
\end{equation}

\label{lagrange statement}
In this section, we will analyse the distance of $\boldsymbol{g_v}$ within several steps $k$, which is given by,
\begin{equation}
    \boldsymbol{g_v} = \boldsymbol{\hat{\epsilon}}\nabla_{\boldsymbol{w}}^{2}\mathcal{L}_{S}(\boldsymbol{w}) - \lambda_0\bigtriangledown_{\boldsymbol{w}}\mathcal{L_S}(\boldsymbol{w})
    \label{g_v eq}
\end{equation}\\
Note that $\boldsymbol{g_v}$ is vertical to the SGD gradient on original weight $\boldsymbol{w}$, which gives us the following constraint:
\begin{equation}
    (\boldsymbol{\hat{\epsilon}}\nabla_{\boldsymbol{w}}^{2}\mathcal{L}_{S}(\boldsymbol{w}) - \lambda_0\bigtriangledown_{\boldsymbol{w}}\mathcal{L_S}(\boldsymbol{w}))\bigtriangledown_{\boldsymbol{w}}\mathcal{L_S}(\boldsymbol{w})=0
\end{equation}
\\
\begin{table*}[ht]

\caption{Architectures of ViTs}
\label{tab_architecture}
\begin{center}
\begin{tabular}{lcccccc}
\multicolumn{1}{c}{\bf Model}
&\multicolumn{1}{c}{\bf Params} &\multicolumn{1}{c}{\bf Patch Resolution} 
&\multicolumn{1}{c}{\bf Sequence Length}
&\multicolumn{1}{c}{\bf Hidden Size} 
&\multicolumn{1}{c}{\bf Heads}
&\multicolumn{1}{c}{\bf Layers}
\\ \hline \\
ViT-B-16       & 87M    & $16 \times 16$   & 196  & 768 & 12 & 12    \\
ViT-B-32       & 88M    & $32 \times 32$  & 49 & 768 & 12 & 12   \\
ViT-S-16       & 22M     & $16 \times 16$     & 196 & 384 & 6 & 12  \\
ViT-S-32       & 23M      & $32 \times 32$    & 49 & 384 & 6 & 12  \\
\\ \hline
\end{tabular}
\end{center}
\end{table*}

We can use Lagrange function to describe and solve this problem. Let's introduce a new variable $\lambda$ and define the Lagrangian $L$ with all the parameters $\boldsymbol{w}$, $\lambda_0$ and $\lambda$ as variable as follows,
\begin{equation}
\begin{aligned}
     L(\boldsymbol{w}, \lambda_0, \lambda) &= \boldsymbol{\hat{\epsilon}} \nabla_{\boldsymbol{w}}^{2}\mathcal{L}_{S}(\boldsymbol{w}) - \lambda_0\bigtriangledown_{\boldsymbol{w}}\mathcal{L_S}(\boldsymbol{w})\\
     &+\lambda(\boldsymbol{\hat{\epsilon}}\nabla_{\boldsymbol{w}}^{2}\mathcal{L}_{S}(\boldsymbol{w})\\
     &- \lambda_0\bigtriangledown_{\boldsymbol{w}}\mathcal{L_S}(\boldsymbol{w}))\bigtriangledown_{\boldsymbol{w}}\mathcal{L_S}(\boldsymbol{w})
\end{aligned}
    \label{lagrange function}
\end{equation}
The partial derivative of $L$ with respect to $\boldsymbol{w}$, $\lambda_0$ and $\lambda$ are as follows,\\
\begin{equation}
\begin{aligned}
    L_{\boldsymbol{w}}(\boldsymbol{w},\lambda_0,\lambda)&=
    \frac{\partial\boldsymbol{\hat{\epsilon}}}{\partial{\boldsymbol{w}}}\bigtriangledown_{\boldsymbol{w}}^2\mathcal{L_S}(\boldsymbol{w})+\boldsymbol{\hat{\epsilon}}\bigtriangledown_{\boldsymbol{w}}^3\mathcal{L_S}(\boldsymbol{w})\\
    &-\lambda_0\bigtriangledown_{\boldsymbol{w}}^2\mathcal{L_S}(\boldsymbol{w})
    +\lambda(\frac{\partial\boldsymbol{\hat{\epsilon}}}{\partial{\boldsymbol{w}}}\bigtriangledown_{\boldsymbol{w}}^2\mathcal{L_S}(\boldsymbol{w})\\
    &+\boldsymbol{\hat{\epsilon}}\bigtriangledown_{\boldsymbol{w}}^3\mathcal{L_S}(\boldsymbol{w})\\
    &-\lambda_0\bigtriangledown_{\boldsymbol{w}}^2\mathcal{L_S}(\boldsymbol{w}) )\bigtriangledown_{\boldsymbol{w}}\mathcal{L_S}(\boldsymbol{w})\\
    &+\lambda(\boldsymbol{\hat{\epsilon}}\bigtriangledown^2_{\boldsymbol{w}}\mathcal{L_S}(\boldsymbol{w})\\
    &-\lambda_0\bigtriangledown_{\boldsymbol{w}}\mathcal{L_S}(\boldsymbol{w}))\bigtriangledown^2_{\boldsymbol{w}}\mathcal{L_S}(\boldsymbol{w})
\end{aligned}
\label{lagrange_w}
\end{equation}
\\
\begin{equation}
    L_{\lambda_0}(\boldsymbol{w},\lambda_0,\lambda)=-\bigtriangledown_{\boldsymbol{w}}\mathcal{L_S}(\boldsymbol{w})
    \label{lagrange_lambda_0}
\end{equation}
\begin{equation}
    L_\lambda(\boldsymbol{w},\lambda_0,\lambda)=(\boldsymbol{\hat{\epsilon}}\nabla_{\boldsymbol{w}}^{2}\mathcal{L}_{S}(\boldsymbol{w})- \lambda_0\bigtriangledown_{\boldsymbol{w}}\mathcal{L_S}(\boldsymbol{w}))\bigtriangledown_{\boldsymbol{w}}\mathcal{L_S}(\boldsymbol{w})
    \label{lagrange_lambda}
\end{equation}
\\
We omit some high order terms that are trivial compared to the first-order terms in calculation and set all of the partial derivatives equal to 0. This gives,
\begin{equation}
\begin{aligned}
    & L_{\boldsymbol{w}}(\boldsymbol{w},\lambda_0,\lambda)\approx\frac{\partial\boldsymbol{\hat{\epsilon}}}{\partial{\boldsymbol{w}}}\bigtriangledown_{\boldsymbol{w}}^2\mathcal{L_S}(\boldsymbol{w})-\lambda_0\bigtriangledown_{\boldsymbol{w}}^2\mathcal{L_S}(\boldsymbol{w})=0 
    \\
    & L_{\lambda_0}(\boldsymbol{w},\lambda_0,\lambda)=-\bigtriangledown_{\boldsymbol{w}}\mathcal{L_S}(\boldsymbol{w}) = 0
    \\ 
    & L_\lambda(\boldsymbol{w},\lambda_0,\lambda)=(\boldsymbol{\hat{\epsilon}}\nabla_{\boldsymbol{w}}^{2}\mathcal{L}_{S}(\boldsymbol{w})- \lambda_0\bigtriangledown_{\boldsymbol{w}}\mathcal{L_S}(\boldsymbol{w}))\bigtriangledown_{\boldsymbol{w}}\mathcal{L_S}(\boldsymbol{w})=0
\end{aligned}
\label{lagrange_partial derivative}
\end{equation}
\\
\\
Since $\frac{\partial\boldsymbol{\hat{\epsilon}}}{\partial{\boldsymbol{w}}}=\rho(\frac{\bigtriangledown_{\boldsymbol{w}}^2\mathcal{L_S}(\boldsymbol{w})}{||\bigtriangledown_{\boldsymbol{w}}\mathcal{L_S}(\boldsymbol{w})||}
-\frac{\bigtriangledown_{\boldsymbol{w}}\mathcal{L_S}(\boldsymbol{w})\bigtriangledown_{\boldsymbol{w}}^2\mathcal{L_S}(\boldsymbol{w})\bigtriangledown_{\boldsymbol{w}}\mathcal{L_S}(\boldsymbol{w})}{||\bigtriangledown_{\boldsymbol{w}}\mathcal{L_S}(\boldsymbol{w})||^3})$ and from Equation \ref{lagrange_partial derivative} we have $\frac{\partial\boldsymbol{\hat{\epsilon}}}{\partial{\boldsymbol{w}}}=\lambda_0$, then\\
\begin{equation}
    \lambda_0 = \rho(\frac{\bigtriangledown_{\boldsymbol{w}}^2\mathcal{L_S}(\boldsymbol{w})}{||\bigtriangledown_{\boldsymbol{w}}\mathcal{L_S}(\boldsymbol{w})||}
    -\frac{\bigtriangledown_{\boldsymbol{w}}\mathcal{L_S}(\boldsymbol{w})\bigtriangledown_{\boldsymbol{w}}^2\mathcal{L_S}(\boldsymbol{w})\bigtriangledown_{\boldsymbol{w}}\mathcal{L_S}(\boldsymbol{w})}{||\bigtriangledown_{\boldsymbol{w}}\mathcal{L_S}(\boldsymbol{w})||^3})
    \label{Lagrange_lambda}
\end{equation}
\\
In addition, from Equation \ref{lagrange_partial derivative}, we have,
\\
\begin{equation}
    (\boldsymbol{\hat{\epsilon}}\nabla_{\boldsymbol{w}}^{2}\mathcal{L}_{S}(\boldsymbol{w}) - \lambda_0\bigtriangledown_{\boldsymbol{w}}\mathcal{L_S}(\boldsymbol{w}))\bigtriangledown_{\boldsymbol{w}}\mathcal{L_S}(\boldsymbol{w})=0
\end{equation}
\\
Then $\lambda_0$ can be written as:
\begin{equation}
    \lambda_0=\rho\frac{\bigtriangledown_{\boldsymbol{w}}\mathcal{L_S}(\boldsymbol{w})\bigtriangledown_{\boldsymbol{w}}^2\mathcal{L_S}(\boldsymbol{w})\bigtriangledown_{\boldsymbol{w}}\mathcal{L_S}(\boldsymbol{w})}{||\bigtriangledown_{\boldsymbol{w}}\mathcal{L_S}(\boldsymbol{w})||^3}
    \label{vertical_lambda_0}
\end{equation}
\\
\\
With Equation \ref{Lagrange_lambda} and \ref{vertical_lambda_0},\\
\begin{equation}
\begin{aligned}
    &\rho(\frac{\bigtriangledown_{\boldsymbol{w}}^2\mathcal{L_S}(\boldsymbol{w})}{||\bigtriangledown_{\boldsymbol{w}}\mathcal{L_S}(\boldsymbol{w})||}
    -\frac{\bigtriangledown_{\boldsymbol{w}}\mathcal{L_S}(\boldsymbol{w})\bigtriangledown_{\boldsymbol{w}}^2\mathcal{L_S}(\boldsymbol{w})\bigtriangledown_{\boldsymbol{w}}\mathcal{L_S}(\boldsymbol{w})}{||\bigtriangledown_{\boldsymbol{w}}\mathcal{L_S}(\boldsymbol{w})||^3})\\
    =&\rho\frac{\bigtriangledown_{\boldsymbol{w}}\mathcal{L_S}(\boldsymbol{w})\bigtriangledown_{\boldsymbol{w}}^2\mathcal{L_S}(\boldsymbol{w})\bigtriangledown_{\boldsymbol{w}}\mathcal{L_S}(\boldsymbol{w})}{||\bigtriangledown_{\boldsymbol{w}}\mathcal{L_S}(\boldsymbol{w})||^3}
    \end{aligned}
\end{equation}
Then we have
\begin{equation}
\frac{1}{2}\rho\frac{\bigtriangledown_{\boldsymbol{w}}^2\mathcal{L_S}(\boldsymbol{w})}{||\bigtriangledown_{\boldsymbol{w}}\mathcal{L_S}(\boldsymbol{w})||} =\rho\frac{\bigtriangledown_{\boldsymbol{w}}\mathcal{L_S}(\boldsymbol{w})\bigtriangledown_{\boldsymbol{w}}^2\mathcal{L_S}(\boldsymbol{w})\bigtriangledown_{\boldsymbol{w}}\mathcal{L_S}(\boldsymbol{w})}{||\bigtriangledown_{\boldsymbol{w}}\mathcal{L_S}(\boldsymbol{w})||^3}
\label{relationship_lambda_0}
\end{equation}
\\
Using the relationship in Equation \ref{relationship_lambda_0}, we can write $\lambda_0$ in this way:\\
\begin{equation}
    \lambda_0=\frac{1}{2}\rho\frac{\bigtriangledown_{\boldsymbol{w}}^2\mathcal{L_S}(\boldsymbol{w})}{||\bigtriangledown_{\boldsymbol{w}}\mathcal{L_S}(\boldsymbol{w})||}
\end{equation}\\
Substituting the value of $\lambda_0$ back to Equation \ref{lagrange function} gives us the maximum of the Lagrangian.\\
\begin{equation}
\begin{aligned}
    \tilde{L}(\boldsymbol{w},\lambda_0,\lambda)&= \boldsymbol{\hat{\epsilon}}\nabla_{\boldsymbol{w}}^{2}\mathcal{L}_{S}(\boldsymbol{w}) - \lambda_0\bigtriangledown_{\boldsymbol{w}}\mathcal{L_S}(\boldsymbol{w})\\
    &=\rho\frac{\bigtriangledown_{\boldsymbol{w}}\mathcal{L_S}(\boldsymbol{w})}{||\bigtriangledown_{\boldsymbol{w}}\mathcal{L_S}(\boldsymbol{w})||}\bigtriangledown_{\boldsymbol{w}}^2\mathcal{L_S}(\boldsymbol{w})\\
    &-\frac{1}{2}\rho\frac{\bigtriangledown_{\boldsymbol{w}}^2\mathcal{L_S}(\boldsymbol{w})}{||\bigtriangledown_{\boldsymbol{w}}\mathcal{L_S}(\boldsymbol{w})||}\bigtriangledown_{\boldsymbol{w}}\mathcal{L_S}(\boldsymbol{w})\\
\end{aligned}
\label{lagrange_max}
\end{equation}
Using this relationship in Equation \ref{lagrange_max} gives us:
\begin{equation}
    \boldsymbol{g_v}=\tilde{L}(\boldsymbol{w},\lambda_0,\lambda)=\frac{1}{2}\rho   \bigtriangledown_{\boldsymbol{w}}^2\mathcal{L_S}(\boldsymbol{w})\frac{\bigtriangledown_{\boldsymbol{w}}\mathcal{L_S}(\boldsymbol{w})}{||\bigtriangledown_{\boldsymbol{w}}\mathcal{L_S}(\boldsymbol{w})||}
    \label{g_v bound}
\end{equation}\\
We can use Equation \ref{g_v bound} to derive the distance of $\boldsymbol{g_v}$ within $k$ steps,\\
\begin{equation}
\begin{aligned}
||\boldsymbol{g_{v,t}} &- \boldsymbol{g_{v,t+k}}|| =||\tilde{L}_{t}(\boldsymbol{w_t},\lambda_0,\lambda) - \tilde{L}_{t+k}(\boldsymbol{w_{t+k}},\lambda_0,\lambda)||\\
&\approx ||\frac{1}{2}\rho \bigtriangledown_{\boldsymbol{w_t}}^{2}\mathcal{L_S}(\boldsymbol{w_t}) \frac{\bigtriangledown_{\boldsymbol{w_t}}\mathcal{L_S}(\boldsymbol{w_t})}{||\bigtriangledown_{\boldsymbol{w_t}}\mathcal{L_S}(\boldsymbol{w_t})||} \\
&- \frac{1}{2}\rho \bigtriangledown_{\boldsymbol{w_{t+k}}}^{2}\mathcal{L_S}(\boldsymbol{w_{t+k}}) \frac{\bigtriangledown_{\boldsymbol{w_{t+k}}}\mathcal{L_S}(\boldsymbol{w_{t+k}})}{||\bigtriangledown_{\boldsymbol{w_{t+k}}}\mathcal{L_S}(\boldsymbol{w_{t+k}})||}||\\
\end{aligned}
    \label{g_proj bound}
\end{equation}

\subsection{LayerSAM \& LookLayerSAM}
\label{psedocode Layersam}
\begin{algorithm}[H]
   \caption{Layer-wise SAM (LayerSAM)}
   \label{alg:FastAdv}
    \begin{algorithmic}
   \STATE {\bfseries Input:} $x \in \mathbb{R}^{d}$, learning rate $\eta_{t}$, update frequency $k$.
   \FOR{$t \leftarrow 1$ {\bfseries to} $T$}
   \STATE Sample Minibatch $\mathcal{B}=\{(x_i,y_i), \cdots, (x_{|\mathcal{B}|}, y_{|\mathcal{B}|})\}$ from $X$.
   \STATE Compute gradient $g = \nabla_{\boldsymbol{w}}L_{B}(\boldsymbol{w})$ on minibatch $\mathcal{B}$.
   \STATE Compute $\boldsymbol{\epsilon^{(i)}} = \rho \ \frac{||\boldsymbol{w}^{(i)}||}{||\boldsymbol{g}^{(i)}||} \cdot
   \nabla_{\boldsymbol{w}}\mathcal{L}_{S}(\boldsymbol{w}) /  \|\nabla_{\boldsymbol{w}}\mathcal{L}_{S}(\boldsymbol{w})\|$
   \STATE Compute gradient approximation for the SAM objective: $\boldsymbol{g_{s}} =  \nabla_{\boldsymbol{w}}L_{\mathcal{B}}(\boldsymbol{w})|_{\boldsymbol{w}+\boldsymbol{\epsilon}}$
   \STATE Update weights: $\boldsymbol{w_{t+1}^{(i)}} = \boldsymbol{w_{t}^{(i)}} - \eta_{t}^{(i)} \cdot \boldsymbol{g_{s}^{(i)}}$
   \ENDFOR
\end{algorithmic}
\end{algorithm}

\begin{algorithm}[h]
   \caption{Look-LayerSAM}
   \label{alg:FastAdv2}
    \begin{algorithmic}
   \STATE {\bfseries Input:} $x \in \mathbb{R}^{d}$, learning rate $\eta_{t}$, update frequency $k$.
   \FOR{$t \leftarrow 1$ {\bfseries to} $T$}
   \STATE Sample Minibatch $\mathcal{B}=\{(x_i,y_i), \cdots, (x_{|\mathcal{B}|}, y_{|\mathcal{B}|})\}$ from $X$.
   \STATE Compute gradient $\boldsymbol{g} = \nabla_{\boldsymbol{w}}L_{B}(\boldsymbol{w})$ on minibatch $\mathcal{B}$.
   \IF{t\%k = 0}
   \STATE  $\epsilon(\boldsymbol{w})^{(i)} = \rho \ \frac{\|\boldsymbol{w}^{(i)}\|}{\|\boldsymbol{g}^{(i)}\|} \cdot
   \nabla_{\boldsymbol{w}}\mathcal{L}_{S}(\boldsymbol{w}) /  \|\nabla_{\boldsymbol{w}}\mathcal{L}_{S}(\boldsymbol{w})\|$
   \STATE Compute     SAM gradient: $\boldsymbol{g_{s}} =  \nabla_{\boldsymbol{w}}L_{\mathcal{B}}(\boldsymbol{w})|_{\boldsymbol{w}+\epsilon(\boldsymbol{w})}$
   \STATE $\boldsymbol{g_v} =  \boldsymbol{g_s} - \|\boldsymbol{g_s}\|  \cos(\theta) \cdot \frac{\boldsymbol{g}}{\|\boldsymbol{g}\|}$, where $\cos(\theta) =\frac{\boldsymbol{g}\cdot\boldsymbol{g_s}}{\|\boldsymbol{g}\|\|\boldsymbol{g_s}\|} $
   \ELSE
   \STATE $\boldsymbol{g_{s}} = \boldsymbol{g} + \alpha \cdot \frac{\|\boldsymbol{g}\|}{\|\boldsymbol{g_{v}}\|} \cdot \boldsymbol{g_{v}}$
   \ENDIF
   \STATE Update weights: $\boldsymbol{w_{t+1}}^{(i)} = \boldsymbol{w_{t}}^{(i)} - \eta_{t}^{(i)} \cdot \boldsymbol{g_{s}}^{(i)}$
   \ENDFOR
\end{algorithmic}
\end{algorithm}



\subsection{Parameter Settings}
\label{parameter_setting}
\begin{table*}[h]\scriptsize
\caption{Parameter Settings of ViT for Vanilla Training}
\label{tab_param}
\begin{center}
\begin{tabular}{lccccccccccc}
\multicolumn{1}{c}{\bf Model}
&\multicolumn{1}{c}{\bf \makecell[c]{Input \\ Resolution}} &\multicolumn{1}{c}{\bf \makecell[c]{Batch \\ Size}} 
&\multicolumn{1}{c}{\bf Epoch}
&\multicolumn{1}{c}{\bf \makecell[c]{Warmup \\ Steps}} 
&\multicolumn{1}{c}{\bf \makecell[c]{Peak \\ LR}}
&\multicolumn{1}{c}{\bf \makecell[c]{LR \\ Decay}} 
&\multicolumn{1}{c}{\bf Optimizer}
&\multicolumn{1}{c}{\bf $\rho$} 
&\multicolumn{1}{c}{\bf \makecell[c]{Weight \\ Decay}} 
&\multicolumn{1}{c}{\bf \makecell[c]{Gradient \\ Clipping}}
\\ \hline \\
ViT-B-16       & 224    & 4096   & 300  & 10000 & 3e-3 & cosine & AdamW & / & 0.3 & 1.0    \\
ViT-B-32       & 224    & 4096  & 300 & 10000 & 3e-3 & cosine & AdamW & / & 0.3 & 1.0   \\
ViT-S-16       & 224     & 4096     & 300 & 10000 & 3e-3 & cosine & AdamW & / & 0.3 & 1.0  \\
ViT-S-32       & 224      & 4096    & 300 & 10000 & 3e-3 & cosine & AdamW & / & 0.3 & 1.0  \\
\\ \hline \\
ViT-B-16 + SAM       & 224    & 4096   & 300  & 10000 & 3e-3 & cosine & AdamW & 0.18 & 0.3 & 1.0    \\
ViT-B-32 + SAM       & 224    & 4096  & 300 & 10000 & 3e-3 & cosine & AdamW & 0.15 & 0.3 & 1.0   \\
ViT-S-16 + SAM       & 224     & 4096     & 300 & 10000 & 3e-3 & cosine & AdamW & 0.1 & 0.3 & 1.0  \\
ViT-S-32 + SAM       & 224      & 4096    & 300 & 10000 & 3e-3 & cosine & AdamW & 0.05 & 0.3 & 1.0  \\
\\ \hline \\
ViT-B-16 + LookSAM       & 224    & 4096   & 300  & 10000 & 3e-3 & cosine & AdamW & 0.18 & 0.3 & 1.0    \\
ViT-B-32 + LookSAM       & 224    & 4096  & 300 & 10000 & 3e-3 & cosine & AdamW & 0.15 & 0.3 & 1.0   \\
ViT-S-16 + LookSAM       & 224     & 4096     & 300 & 10000 & 3e-3 & cosine & AdamW & 0.1 & 0.3 & 1.0  \\
ViT-S-32 + LookSAM       & 224      & 4096    & 300 & 10000 & 3e-3 & cosine & AdamW & 0.05 & 0.3 & 1.0  \\
\\ \hline
\end{tabular}
\end{center}
\end{table*}
\begin{table*}[h]\scriptsize
\caption{Parameter Settings of ViT for Large-Batch Training}
\label{tab_param_2}
\begin{center}
\begin{tabular}{lcccccccccc}
\multicolumn{1}{c}{\bf Model}
 &\multicolumn{1}{c}{\bf \makecell[c]{Batch \\ Size}} 
&\multicolumn{1}{c}{\bf Epoch}
&\multicolumn{1}{c}{\bf \makecell[c]{Warmup \\ Steps}} 
&\multicolumn{1}{c}{\bf \makecell[c]{Peak \\ LR}}
&\multicolumn{1}{c}{\bf \makecell[c]{LR \\ Decay}} 
&\multicolumn{1}{c}{\bf Optimizer}
&\multicolumn{1}{c}{\bf $\rho$}
&\multicolumn{1}{c}{\bf $\alpha$} 
&\multicolumn{1}{c}{\bf \makecell[c]{Weight \\ Decay}} 
&\multicolumn{1}{c}{\bf \makecell[c]{Gradient \\ Clipping}}
\\ \hline \\
ViT-B-16 + SAM       & 4096   & 300  & 10000 & 1e-2 & linear & LAMB & 0.18 & / & 0.1 & 1.0    \\
ViT-B-16 + SAM       & 8192   & 300  & 10000 & 1.7e-2 & linear & LAMB & 0.18 & / & 0.1 & 1.0    \\
ViT-B-16 + SAM       & 16834   & 300  & 7000 & 1.8e-2 & linear & LAMB & 0.18 & / & 0.1 & 1.0    \\
ViT-B-16 + SAM       & 32768   & 300  & 6000 & 1.8e-2 & linear & LAMB & 0.18 & / & 0.1 & 1.0    \\

\\ \hline \\
ViT-B-16 + LayerSAM       & 4096   & 300  & 10000 & 1e-2 & linear & LAMB & 1.0 & / & 0.1 & 1.0    \\
ViT-B-16 + LayerSAM       & 8192  & 300 & 10000 & 1.7e-2 & linear & LAMB & 1.0 & / & 0.1 & 1.0   \\
ViT-B-16 + LayerSAM       & 16384     & 300 & 7000 & 1.8e-2 & linear & LAMB & 1.0 & / & 0.1 & 1.0  \\
ViT-B-16 + LayerSAM       & 32768    & 300 & 6000 & 1.8e-2 & linear & LAMB & 1.0 & / & 0.1 & 1.0  \\
ViT-B-16 + LayerSAM       & 65536    & 300 & 3500 & 2e-2 & linear & LAMB & 1.0 & / & 0.2 & 1.0  \\
\\ \hline \\
ViT-B-16 + Look-LayerSAM       & 4096   & 300  & 10000 & 1e-2 & linear & LAMB & 1.0 &0.7 & 0.1 & 1.0    \\
ViT-B-16 + Look-LayerSAM       & 8192  & 300 & 10000 & 1.7e-2 & linear & LAMB & 1.0 & 0.7 & 0.1 & 1.0   \\
ViT-B-16 + Look-LayerSAM       & 16384     & 300 & 7000 & 1.8e-2 & linear & LAMB & 1.0 &0.7 & 0.1 & 1.0  \\
ViT-B-16 + Look-LayerSAM       & 32768    & 300 & 6000 & 1.8e-2 & linear & LAMB & 1.0 &0.7 & 0.1 & 1.0  \\
ViT-B-16 + Look-LayerSAM       & 65536    & 300 & 3500 & 2e-2 & linear & LAMB & 1.0 &0.7 & 0.2 & 1.0  \\
\\ \hline
\end{tabular}
\end{center}
\end{table*}
In this section, we will introduce the architectures of ViTs in this paper (Table \ref{tab_architecture}). Next, we provide the hyperparameters in Table \ref{tab_param} for ViT training, including learning rate, warmup, optimizer, gradient clipping, epoch, etc. In addition, Table \ref{tab_param_2} gives us the parameter settings of ViT for large-batch training in this paper.

\subsection{Generalization bound}
\label{analyse_reuse}
We firstly introduce
\label{PAC appendix} Theorem
\ref{bound PAC} regarding generalization bound based on sharpness of LookSAM and then give a proof for it. Note that a similar bound was also established in the original SAM paper~\cite{foret2020sharpness}. 
\def\mathbi#1{\textbf{\em #1}} 
\begin{theorem}
\label{bound PAC}
With probability 1 - $\delta$ over the choice the training set $\mathcal{S} \sim \mathcal{D}$, we have\\
\begin{equation}
    \begin{aligned}
    & \mathcal{L}_{D}(\boldsymbol{w})\leq
    \max_{||\boldsymbol{\epsilon'}||_{p}\leq \rho'}
    \mathcal{L}_{S}(\boldsymbol{w}+\boldsymbol{\epsilon'})
    \\
    &+ \sqrt{\frac{k\log(1+\frac{||\boldsymbol{w}||_2^2}{\rho'^2}
    (1+\sqrt{\frac{\log(n)}{k}})^2)+4\log\frac{n}{\delta}+\tilde{O}(1)}{n-1}}
    \end{aligned}
\end{equation}

where n = $|\mathcal{S}|$ and $\rho'^2=\rho^2+\rho_0^2$.
\end{theorem}

\begin{proof}\renewcommand{\qedsymbol}{}

 We start by illustrating the PAC-Bayesian Generation Bound theorem, which gives a bound on the generalization error of any posterior distribution $\mathcal{Q}$ on parameters that can be achieved using a selected prior distribution $\mathcal{P}$ over parameters training with data set $\mathcal{S}$. Let $KL(\mathcal{Q}||\mathcal{P})$ denote the KL divergence between two Bernoulli distributions $\mathcal{P}$ and $\mathcal{Q}$, we have:
\begin{equation}
     \mathbb{E}_{\boldsymbol{w}\sim\mathcal{L}}[L_\mathcal{D}(\boldsymbol{w})]\leq
     \mathbb{E}_{\boldsymbol{w}\sim\mathcal{L}}[L_\mathcal{S}(\boldsymbol{w})]
     + \sqrt{\frac{KL(\mathcal{Q||\mathcal{P}})+\log\frac{n}{\delta}}{2(n-1)}}
\end{equation}
In order to accelerate the training process, LookSAM calculate the SAM gradient only at every k step and try to reuse the projected components to imitate the weight perturbations introduced from SAM procedure in the subsequent steps. We use $\boldsymbol{\epsilon^0}$ to indicate the difference between our imitated weight perturbation, $\boldsymbol{\epsilon'}$, from LookSAM and the real weight perturbation, $\boldsymbol{\epsilon}$, from SAM. As the optimization is in fact regarding the distribution of $\boldsymbol{\epsilon'}$, we assume that $\mathcal{L}_\mathcal{D}(\boldsymbol{w})\leq\mathbb{E}_{\epsilon'_i\sim\mathcal{N}(0,\rho')}[L_\mathcal{D}(\boldsymbol{w}+\boldsymbol{\epsilon'})]$, which indicates adding Gaussian perturbation should not decrease the test error\cite{foret2020sharpness}. Following \cite{foret2020sharpness}, the generalization bound can be written as follows:\\
\begin{equation}
\begin{aligned}
    &\mathbb{E}_{\epsilon'_i\sim\mathcal{N}(0,\sigma')}[L_\mathcal{D}(\boldsymbol{w}+\boldsymbol{\epsilon'})]\leq
    \mathbb{E}_{\epsilon'_i\sim\mathcal{N}(0,\sigma')}[L_\mathcal{S}(\boldsymbol{w}+\boldsymbol{\epsilon'})]\\
    &+\sqrt{\frac{\frac{1}{4}k\log(1+\frac{||\boldsymbol{w}||_2^2}{k\sigma'^2})+\frac{1}{4}+log\frac{n}{\delta}+2log(6n+3k)}{n-1}},
    \label{LookSAM_bound}
\end{aligned}
\end{equation}
where $\epsilon_i'=\epsilon_i+\epsilon^0_i$ 
\\
\\
 In Equation (\ref{LookSAM_bound}), we assume that $\epsilon_i$ and $\epsilon^0_i$ are independent normal variables with mean 0, and corresponding variance $\sigma$ and $\sigma_0$ respectively. Let $\{\epsilon'_i\}$, where $\epsilon'_i=\epsilon_i+\epsilon^0_i$, be the independent normal variable with mean 0 and variance $\sigma'^2=\sigma^2+\sigma_0^2$. In particular, at the time when LookSAM can perfectly imitate the SAM procedure by reusing the projected gradient, $\sigma_0^2$ becomes zero and $\sigma'^2$ equals to $\sigma^2$. As $||\boldsymbol{\epsilon'}||_2^2$ has chi-square distribution in this case and based on concentration inequality from Lemma 1 in \cite{laurent2000adaptive}, we obtain the following for any positive $x$:
\\
\begin{equation}
\begin{aligned}
     P(&||\boldsymbol{\epsilon}+\boldsymbol{\epsilon_0}||_2^2-k(\sigma^2+\sigma_0^2)\\
     &\geq2(\sigma^2+\sigma_0^2)\sqrt{k x}+2x(\sigma^2+\sigma_0^2))\\
     &\leq exp(-x)   
\end{aligned}
\end{equation}

Let $x=\ln{\sqrt{n}}$, then we have that
\begin{equation}
\begin{aligned}
        P(&||\boldsymbol{\epsilon}+\boldsymbol{\epsilon_0}||_2^2\\
        &\geq(\sigma^2+\sigma_0^2)(k+2\sqrt{k\ln{\sqrt{n}}}+2\ln{\sqrt{n}}))\\&\leq
    \frac{1}{\sqrt{n}}
\end{aligned}
\end{equation}
\\
With probability of $(1-\frac{1}{\sqrt{n}})$, we have,
\\
\begin{equation}
\begin{aligned}
    ||\boldsymbol{\epsilon'}||_2^2 &= ||\boldsymbol{\epsilon}+\boldsymbol{\epsilon_0}||_2^2 \\&\leq
    (\sigma^2+\sigma_0^2)(k+2\sqrt{k\ln{\sqrt{n}}}+2\ln{\sqrt{n}})\\&\leq
    (\sigma^2+\sigma_0^2)k(1+\sqrt{\frac{\ln{n}}{k}})^2\\&\leq
    \rho^2 + \rho_0^2 ,
\end{aligned}
\end{equation}
where $\rho_0^2=\sigma_0^2k(1+\sqrt{\frac{\ln{n}}{k}})^2$.
\\
\\
 After substituting the value for $\sigma'$ back to Equation (\ref{LookSAM_bound}), we can generate the following bounds:
 \begin{equation}
\label{looksam_final_bound}
\begin{aligned}
    &\mathcal{L}_\mathcal{D}(\boldsymbol{w})\leq
    (1-\frac{1}{\sqrt{n}})\max_{||\boldsymbol{\epsilon'}||_{{p}\leq \rho'}}
    \mathcal{L}_{S}(\boldsymbol{w}+\boldsymbol{\epsilon'})+
    \frac{1}{\sqrt{n}}\\
    &+\sqrt{\frac{\frac{1}{4}k\log(1+\frac{||\boldsymbol{w}||_2^2}{\rho^2}
    (1+\sqrt{\frac{\log(n)}{k})})^2+\log\frac{n}{\delta}+2\log(6n+3k))}{n-1}}\\
    & \leq\max_{||\boldsymbol{\epsilon'}||_{{p}\leq \rho'}} \mathcal{L}_{S}(\boldsymbol{w}+\boldsymbol{\epsilon'})\\
    &+\sqrt{\frac{k\log(1+\frac{||\boldsymbol{w}||_2^2}{\rho'^2}
    (1+\sqrt{\frac{\log(n)}{k}})^2)+4\log\frac{n}{\delta}+8\log(6n+3k))}{n-1}}
    \end{aligned}
\end{equation}
where $\rho'^2=\rho^2+\rho_0^2$.

\end{proof}


\end{document}